%% file: main.tex
\RequirePackage[svgnames]{xcolor}

\documentclass[11pt,letterpaper]{mystyle}
\usepackage{amsmath}
\usepackage{amssymb}
\usepackage[all]{hypcap}
\usepackage[svgnames]{xcolor}
\usepackage[comma,authoryear,compress]{natbib}
\usepackage{hyperref}[citecolor=lightblue]

\hypersetup{
    colorlinks = true,
    citecolor = {violet},
}

\usepackage{algorithm}
\usepackage{algorithmicx}
\usepackage{algpseudocode}
\usepackage{microtype}
\usepackage{graphicx}
\expandafter\def\csname ver@subfig.sty\endcsname{}
\usepackage{booktabs} %
\usepackage{float}
\usepackage{bigstrut}

\usepackage{mathtools}
\usepackage{amsthm}
\usepackage{mathrsfs}
\usepackage{nicefrac}
\usepackage{dsfont}
\usepackage{enumitem}
\usepackage{subcaption}
\usepackage{graphicx,subfig}
\usepackage{cleveref}
\usepackage{bxcoloremoji}
\usepackage{datetime2}
\usepackage{float}
\usepackage{physics}

\usepackage[utf8]{inputenc} %
\usepackage[T1]{fontenc}    %
\usepackage{hyperref}       %
\usepackage{url}            %
\usepackage{booktabs}       %
\usepackage{amsfonts}       %
\usepackage{nicefrac}       %
\usepackage{microtype}      %
\usepackage{graphicx}
\usepackage{subcaption} 
\usepackage{fdsymbol}
\usepackage{wrapfig}
\usepackage{lipsum}
\usepackage{enumitem}
\usepackage{stackengine}
\usepackage[font=small,labelfont=bf]{caption}
\usepackage{color}
\usepackage{adjustbox}
\usepackage{multirow, multicol}

\usepackage{rotating}
\usepackage{makecell}

\usepackage{tcolorbox}
\tcbuselibrary{listings,skins,breakable}
\usepackage{listings}
\usepackage{subcaption}

\def\P{\mathbb{P}}

\newcommand{\given}{\,|\,}
\renewcommand{\hat}{\widehat}

\input{macro}

\newtcolorbox{AIbox}[2][]{aibox,title=#2,#1}
\definecolor{lightgreen}{rgb}{0.22,0.70,0.30}%
\definecolor{Gray}{gray}{0.95}
\definecolor{Cornsilk}{rgb}{1.0, 0.97, 0.86}

\definecolor{lightblue}{HTML}{0064E0}
\definecolor{fg}{HTML}{1C2B33}
\definecolor{bg}{HTML}{F1F4F7}

\newcommand{\name}{\textsc{Rubric-ARM}}
\newcommand{\modelname}{\textsc{Rubric-ARM}}

\usepackage{amsmath}

\usepackage[all]{hypcap}
\title{\LARGE  Alternating Reinforcement Learning for Rubric-Based Reward Modeling in Non-Verifiable LLM Post-Training}

\runningtitle{Alternating Reinforcement Learning for Rubric-Based Reward Modeling}

\author{
Ran Xu\textsuperscript{1,*} \quad
Tianci Liu\textsuperscript{2,*} \quad
Zihan Dong\textsuperscript{3} \quad
Tony Yu\textsuperscript{4} \quad
Ilgee Hong\textsuperscript{4} \\
\bf Carl Yang\textsuperscript{1} \quad
Linjun Zhang\textsuperscript{3} \quad
Tuo Zhao\textsuperscript{4} \quad
Haoyu Wang\textsuperscript{5} \\
\textsuperscript{1}Emory University \quad
\textsuperscript{2}Purdue University \quad
\textsuperscript{3}Rutgers University \quad

\textsuperscript{4}Georgia Institute of Technology \quad
\textsuperscript{5}University at Albany
}

\begin{document}

\input{sections/abstract}
\maketitle
\def\thefootnote{$^{*}$}\footnotetext{These authors contributed equally to this work, order was determined randomly (by rolling a die).}\def\thefootnote{\arabic{footnote}}
\vspace{3mm}
\input{sections/introduction}
\input{sections/relatedwork}

\input{sections/method}

\input{sections/theory}

\input{sections/experiments}

\input{sections/conclusion}
\bibliography{main}

\appendix
\input{sections/appendix}
\appendix
\end{document}

%% file: macro.tex
\newcommand{\x}{\mathbf{x}}

\DeclareMathOperator*{\E}{\mathbb{E}}

\definecolor{blanchedalmond}{rgb}{1.0, 0.92, 0.8}
\definecolor{carmine}{rgb}{0.59, 0.0, 0.09}
\definecolor{lightblue}{rgb}{0.22,0.45,0.70}%

\newtheorem{theorem}{Theorem}[section]

\newtheorem{lemma}[theorem]{Lemma}

\newtheorem{proposition}[theorem]{Proposition}

\newtheorem{remark}[theorem]{Remark}

\newtheorem{assumption}[theorem]{Assumption}

\renewcommand{\mathbf}{\boldsymbol}

\renewcommand{\P}{\mathbb{P}}

\makeatletter
\def\Ddots{\mathinner{\mkern1mu\raise\p@
\vbox{\kern7\p@\hbox{.}}\mkern2mu
\raise4\p@\hbox{.}\mkern2mu\raise7\p@\hbox{.}\mkern1mu}}
\makeatother

\definecolor{amaranth}{rgb}{0.9, 0.17, 0.31}
\definecolor{antiquebrass}{rgb}{0.8, 0.58, 0.46}
\definecolor{antiquefuchsia}{rgb}{0.57, 0.36, 0.51}
\definecolor{chromeyellow}{rgb}{0.31, 0.47, 0.26}

%% file: sections/abstract.tex
\begin{abstract}
Standard reward models typically predict scalar scores that fail to capture the multifaceted nature of response quality in non-verifiable domains, such as creative writing or open-ended instruction following. To address this limitation, we propose \modelname{}, a framework that jointly optimizes a rubric generator and a judge using reinforcement learning from preference feedback. Unlike existing methods that rely on static rubrics or disjoint training pipelines, our approach treats rubric generation as a latent action learned to maximize judgment accuracy. We introduce an alternating optimization strategy to mitigate the non-stationarity of simultaneous updates, providing theoretical analysis that demonstrates how this schedule reduces gradient variance during training. Extensive experiments show that \modelname{} achieves strong performance among baselines on multiple benchmarks and significantly improves downstream policy alignment in both offline and online reinforcement learning settings.
\vspace{2mm}


\vspace{5mm}

\coloremojicode{1F4C5} \textbf{Date}: \today



\coloremojicode{1F917} \textbf{Model Weights \& Checkpoints}: \href{https://huggingface.co/collections/OpenRubrics/rubricarm}{https://huggingface.co/collections/OpenRubrics/rubricarm}


\coloremojicode{1F4E7} \textbf{Contact}: 
\href{mailto:ran.xu@emory.edu}{ran.xu@emory.edu}; \href{mailto:liu3351@purdue.edu}{liu3351@purdue.edu}; \href{mailto:hwang28@albany.edu}{hwang28@albany.edu}

\end{abstract}

%% file: sections/introduction.tex
\section{Introduction}
Reward modeling serves as the compass for aligning large language models (LLMs) with human intents, typically by generating a scalar score or preference label to predict human preferences~\citep{stiennon2020learning,wang2024secrets}. However, in complex non-verifiable domain, such as creative writing or open-ended instruction following, these scalar or pairwise judgments often fail to capture the multifaceted nature of response quality~\citep{ying2025beyond}.
To address this limitation, recent advancements have shifted toward rubric-based reward modeling, where models explicitly generate structured criteria to ground their judgments~\citep{gunjal2025rubrics,liu2025openrubrics,pathak2025rubric}.
By decomposing evaluation into interpretable dimensions, rubric-based models offer transparency and improve generalization across prompt-specific evaluation axes.

Central to rubric-based evaluation is the availability of \emph{high-quality rubrics}.
To ensure rubric quality, earlier work has primarily relied on human-authored rubrics, which are expensive to produce and difficult to scale to large datasets~\citep{arora2025healthbench}.
More recent approaches seek to automate rubric construction using LLMs~\citep{viswanathan2025checklists,gunjal2025rubrics}; however, these methods are largely prompting-based and rely on fixed, frozen models for both rubric generation and response quality judgment.
Consequently, they do not update the model’s intrinsic capabilities to the target domain or the underlying preference distribution, limiting their ability to generate in-domain, preference-aligned rubrics. 
Moreover, even when learning-based components are introduced~\citep{liu2025openrubrics,rezaei2025online}, the rubric generator and the judge are treated as separate modules and trained independently rather than jointly optimized.
This decoupled training pipeline prevents deeper integration between rubric construction and judgment, leading to suboptimal evaluation signals. 
Designing effective rubric-based reward models are still challenging.

In this work, we propose \modelname, an end-to-end framework that jointly optimizes the \emph{rubric generator} and the \emph{judge} via alternating reinforcement learning (RL), enabling the two components to co-evolve and mutually reinforce one another during training. 
We formulate rubrics as \emph{latent actions} that guide the reward model in recovering the underlying preference signal, and posit that improved rubric generation directly leads to more accurate preference predictions.
To ensure stable joint optimization, \modelname{} employs an alternating training strategy that decouples the learning dynamics while preserving a shared objective. Training alternates between (i) optimizing the reward model with a fixed rubric generator to align with target preference labels, and (ii) optimizing the rubric generator with a fixed reward model to produce discriminative rubrics that maximize prediction accuracy. 

A key challenge of the alternating RL is the instability caused by simultaneous updates to both components. Our analysis reveals that early-stage exploration by the rubric generator can dominate the learning dynamics.
To mitigate this, we first stabilize the reward model under fixed rubrics before optimizing the rubric generator. 
This alternating schedule reduces variance and ensures robust optimization.

Our contributions can be summarized as follows:
\begin{itemize}[leftmargin=0.4cm]
    \item We develop \modelname{}, a rubric-based reward model to produce high-quality rubrics and precise judgments. To the best of our knowledge, this is the first approach that jointly optimizes rubric and judging via RL.
    \item We introduce an \emph{alternating RL} training algorithm that couples the rubric generator and judge through a shared correctness objective, enabling mutual improvement while stabilizing optimization. 
    \item We evaluate \modelname{} across diverse alignment settings (9 reward modeling and 6 policy benchmarks). \modelname{} outperforms strong reasoning-based judges and prior rubric-based reward models, achieving a $+4.7\%$ average gain on reward-modeling benchmarks, and consistently improves downstream policy post-training when used as the reward signal.
\end{itemize}

%% file: sections/relatedwork.tex
\section{Related Works}
\label{sec:related}
\textbf{LLM-based Reward and Judge Models.} 
While \citet{zheng2023judging} established the foundational utility of LLM-based judges. 
Subsequent research expanded the scope of reasoning to include chain-of-thoughts \citep{zhang2025generative}, self-critiques~\citep{ankner2024critiqueoutloud,yu2025self,mahan2024generative} or plan evaluations strategically~\citep{evalplanner}. \citet{liu2025inference} explore inference-time reasoning for generative reward models. Recent studies~\citep{chen2025judgelrm,chen2025rmr1,whitehouse2025j1,guo2025reward,hong2025thinkrm,xu2025incentivizing} leverage online RL to directly incentivize detailed reasoning, aiming to mitigate bias and enhance the accuracy of pointwise and pairwise scoring.

\textbf{Rubrics-based Reward Models.} Recently, rubric-based approaches have emerged as a promising direction for LLM evaluation \citep{arora2025healthbench,hashemi2024llm,pathak2025rubric,akyurek2025prbench}, alignment \citep{viswanathan2025checklists,zhang2025chasing}, and reasoning \citep{gunjal2025rubrics,zhou2025breaking,huang2025reinforcement}. 
However, a unique challenge lies in generating \emph{high-quality rubrics at scale}.
To address this, \citet{li2026rubrichub,liu2025openrubrics,xie2025auto} extract rubrics from pairwise comparison signals, while \citet{rezaei2025online,zhang2025chasing,shao2025dr} dynamically generate rubrics by leveraging policy model outputs in an online setting.

%% file: sections/method.tex
\vspace{-0.6ex}
\section{Preliminaries}
\vspace{-0.5ex}
\label{subsec:overview}
We study rubric-based reward modeling in \emph{non-verifiable} domains, where response quality cannot be directly validated against ground truth. 
The rubric-based reward model contains two parts, namely \emph{rubric generator} and \emph{judge}. The key components of \modelname{} are described as follows.

\paragraph{Rubrics.}
We define a rubric as a structured set of evaluation criteria conditioned on a prompt.
Formally, let $x$ denote a prompt, a rubric ${r}(x)=\{c_i\}_{i=1}^k$ consists of $k$ criteria, where each $c_i$ specifies a distinct aspect of response quality (e.g., factual correctness, tone, or presentation).

For training rubric-based reward models in non-verifiable domains, a pairwise preference dataset is given as 
$\mathcal{D}=\{(x_i, y^{(1)}_i, y^{(2)}_i, o^\star_i)\}_{i=1}^{N}$, 
where $x$ is a prompt, $y^{(1)}$ and $y^{(2)}$ are two candidate responses, and $o^\star\in\{0,1\}$ indicates which response is preferred
(e.g., $o^\star=1$ means $y^{(1)} \succ y^{(2)}$).
Formally, the \textbf{rubric generator} $\pi_r$ generates a rubric $r$ from the prompt as 
\begin{equation}
r \sim \pi_r(\cdot \mid x; \theta_r),
\end{equation}
while a \textbf{judge} $\pi_j$ predicts a preference $o$ with the reasoning chain $c$ conditioned on the prompt, responses, and rubric as 
\begin{equation}
(c, o) \sim \pi_j(\cdot \mid x, y^{(1)}, y^{(2)}, r; \theta_j).
\end{equation}

\paragraph{Learning Objective.} We define the preference-correctness reward
\begin{equation}
R(o, o^\star) \;=\; \mathbb{I}[o = o^\star],
\end{equation}
where $\mathbb{I}[o = o^\star]$ represents if the binary prediction extracted from $o$ aligns with ground truth $o^\star$.  

Denote $\theta_r,\theta_j$ as the parameter for $\pi_r$ and $\pi_j$ respectively, our goal is to learn $(\theta_r,\theta_j)$ that maximize expected preference correctness under generated rubrics:
\begin{equation}
\label{eq:joint_objective}
\max_{\theta_r,\theta_j}\;
\E_{(x,y^{(1)},y^{(2)},o^\star)\sim\mathcal{D}}
 \E_{r\sim \pi_r(\cdot\mid x;\theta_r)} 
\E_{(c, o)\sim \pi_j(\cdot\mid x,y^{(1)},y^{(2)},r;\theta_j)}
\big[ R(o,o^\star) \big].
\end{equation}
Since both $r$ (text) and $c, o$ (discrete decision with reasoning) are sampled actions, we optimize \cref{eq:joint_objective} with RL.

\begin{figure}[t!]
  \centering  \includegraphics[width=0.7\linewidth]{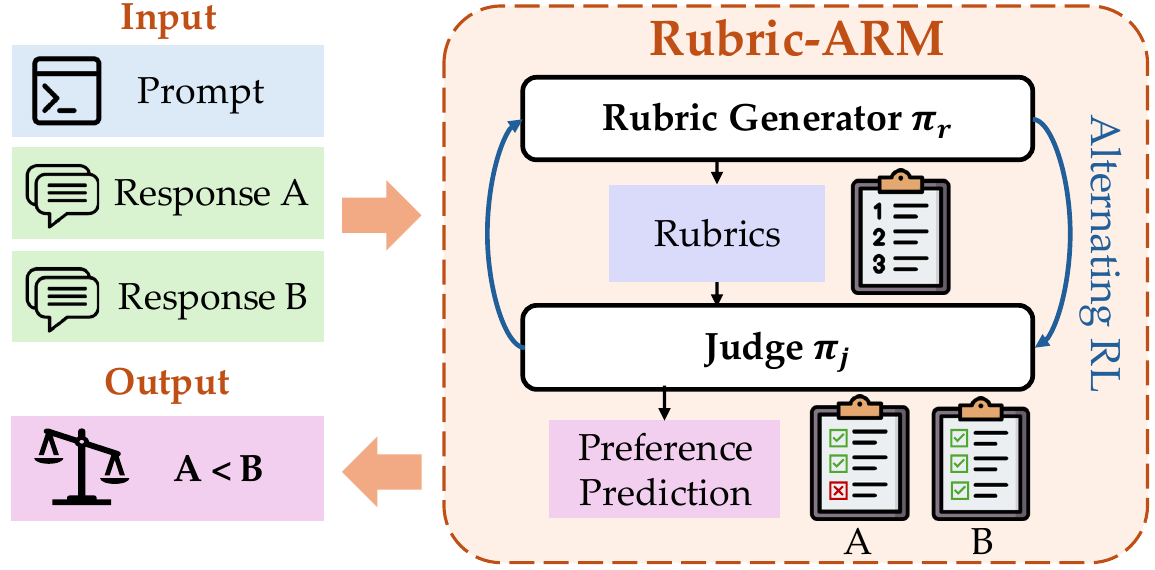}
  \caption{\centering The overall framework for \modelname{}. \vspace{-1.8ex}}
  \label{fig:rl_framework}
\end{figure}

\section{\modelname{}: Alternating RL for Rubric Generation and Judging}
\label{sec:method}
In non-verifiable domains, supervision is limited to pairwise preference feedback and rubrics are not directly observed. Simultaneously updating the rubric generator $\pi_r$ and the judge $\pi_j$ leads to non-stationary learning targets and unstable optimization. 
As shown in Figure \ref{fig:rl_framework}, {\name} addresses this challenge using an alternating RL  scheme that decouples the updates of two components. 
\subsection{Stage I: SFT Warmup} 
We equip both $\pi_j$ and $\pi_{r}$ with basic rubric generation and judging capabilities via leveraging open-source datasets. 
Following the prior work~\citep{liu2025openrubrics}, we fine-tune on synthetic rubrics and judge trajectories derived from open-source datasets including \emph{UltraFeedback}~\citep{cui2024ultrafeedback}, \emph{SkyWork}~\citep{liu2024skywork}, \emph{Magpie}~\citep{xu2025magpie}, and \emph{Synthetic Instruction Following}~\citep{lambert2025tulu}. 
Both $\pi_r(r\mid x; \theta_r)$ and $\pi_j(c,o\mid x,y^{(1)},y^{(2)},r;\theta_j)$ are trained with the standard  next-token prediction objective.

\subsection{Stage II: Alternating Reinforcement Learning}
Stage I (SFT) warm-starts the rubric generator $\pi_r$ and judge $\pi_j$ by imitating synthetic rubric generation and judging trajectories, but optimizes the two components independently and does not directly target preference correctness. 
We therefore optimize both components using \emph{alternating reinforcement learning}. 
Specifically, training switches between (i) \emph{improving the judge with a fixed rubric generator} and (ii) \emph{improving the rubric generator with a fixed judge}, providing each component with a clearer learning signal while preserving the same end objective $R(o,o^\star)$.

\paragraph{(i) RL for Judge $\pi_j$ with the current $\pi_{r}$.}
With the rubric generator parameters $\theta_r$ held fixed, we update $\theta_j$ to improve preference correctness under rubrics sampled from $\pi_r$:
\begin{equation}
\label{eq:judge_objective}
\max_{\theta_j}\;
J_j(\theta_j;\theta_r)
={}
\E_{(x,y^{(1)},y^{(2)},o^\star)\sim\mathcal{D}}
 \E_{r\sim \pi_r(\cdot\mid x;\theta_r)}
\E_{(c, o)\sim \pi_j(\cdot\mid x,y^{(1)},y^{(2)},r;\theta_j)}
\big[\mathbb{I}[o=o^\star]\big].
\end{equation}
This phase trains the judge to produce rubric-conditioned evaluations that recover the dataset preference.

Since $\pi_r(\cdot\mid x;\theta_r)$ is fixed during judge updates, we cache rubrics to reduce sampling cost and stabilize optimization. For each training instance $(x_i,y^{(1)}_i,y^{(2)}_i,o^\star_i)$, we sample a rubric $r_i\sim \pi_r(\cdot\mid x_i;\theta_r)$ once and reuse it for multiple judge optimization steps, yielding the Monte Carlo estimate:
\begin{equation}
\label{eq:judge_objective_cached}
J_j(\theta_j;\theta_r)
\approx{} \E_{(x_i,y^{(1)}_i,y^{(2)}_i,o^\star_i)\sim\mathcal{D},\, r_i} \E_{(c, o)\sim \pi_j(\cdot\mid x_i,y^{(1)}_i,y^{(2)}_i,r_i;\theta_j)}
\big[\mathbb{I}[o=o^\star_i]\big].
\end{equation}
In practice, we use a shaped reward that augments the final correctness signal $R_{\text{acc}}=\mathbb{I}[o=o^\star_i]$ with \emph{format-based} reward $R_{\text{fmt}}$ that enforces valid judging trajectories (i.e., addressing each rubric criterion with per-criterion explanations, followed by an overall justification and a final decision).
The final reward for the judge $\pi_j$ is $R_j = R_{\text{acc}} + R_{\text{fmt}}$.

\noindent \paragraph{(ii)  RL for Rubric Generator $\pi_r$ with the current $\pi_{j}$.}
With the judge parameters $\theta_j$ fixed, we update $\theta_r$ to prefer rubrics that lead the current judge to make correct decisions. Concretely, we maximize the  preference correctness under rubrics drawn from $\pi_r$ as:
\begin{equation}
\label{eq:rubric_objective}
\max_{\theta_r}\;
J_r(\theta_r;\theta_j)
={}\E_{(x,y^{(1)},y^{(2)},o^\star)\sim\mathcal{D}}
\E_{r\sim \pi_r(\cdot\mid x;\theta_r)}
\E_{(c, o)\sim \pi_j(\cdot\mid x,y^{(1)},y^{(2)},r;\theta_j)}
\big[\mathbb{I}[o=o^\star]\big].
\end{equation}
Intuitively, $\pi_r$ learns to generate criteria that are discriminative for the given prompt and usable by the judge to recover the dataset preference.  
In practice, we approximate the expectation with a single rollout by greedy decoding ($t=0$), i.e., we generate one judging trajectory $(c,o)$ per rubric and use the Monte Carlo estimate
\begin{equation}
\label{eq:rubric_reward_mc}
{R}_r = \mathbb{I}[o=o^\star].
\end{equation}
\textbf{Optimization (alternating RL).}
\modelname{} alternates between optimizing~Eq.~\ref{eq:judge_objective} and \ref{eq:rubric_objective}. At iteration $t$, we run:
\begin{align}
r_i^{t} &\sim \pi_r(\cdot\mid x_i;\theta_r^{t}) \quad \forall (x_i,y_i^{(1)},y_i^{(2)},o_i^\star)\in\mathcal{D}, \\
\theta_j^{t+1} &\leftarrow \mathrm{GRPO}\!\left(\theta_j^{t}\,;\,\{r_i^{t}\},\mathcal{D}\right), \\
\theta_r^{t+1} &\leftarrow \mathrm{GRPO}\!\left(\theta_r^{t}\,;\,\theta_j^{t+1},\mathcal{D}\right).
\end{align}
Here we cache one rubric per instance during judge updates (since $\pi_r$ is fixed in that phase). In each phase, GRPO~(\citet{shao2024deepseekmath}, details in Appendix \ref{app:grpo}) updates only the active policy while keeping the other fixed. Notably, we alternate training by updating the judge before the rubric generator in each cycle. In Sec.~\ref{sec:analysis}, we provide theoretical analysis proving the benefits of this ordering.

\paragraph{Connection to EM Algorithm.} Our alternating optimization can be viewed as a generalized EM procedure~\citep{dempster1977maximum} with rubrics $r$ as latent variables. For each preference instance $(x, y^{(1)}, y^{(2)}, o^{\star})$, the judge defines a conditional model $p_{\theta_j}(o^{\star} \mid x, y^{(1)}, y^{(2)}, r)$, while the rubric generator $\pi_r(r \mid x ; \theta_r)$ acts as an amortized variational distribution over the latent rubric~\citep{agrawal2021amortized}. 
With $\pi_r$ fixed, updating $\pi_j$ maximizes the expected correctness (or log-likelihood) under sampled rubrics, analogous to the M-step. With $\pi_j$ fixed, updating $\pi_r$ increases probability mass on rubrics that make the current judge more likely to recover $o^{\star}$, analogous to an amortized E-step. Because rubrics are high-dimensional discrete text sequences, we use stochastic policy-gradient updates rather than exact posterior inference, yielding a stochastic EM-style coordinate ascent scheme.

\subsection{Policy Model Post-training with \modelname{}}
\label{sec:rl}
We use the trained rubric generator $\pi_r(\cdot\mid q;\theta_r)$ and judge $\pi_j(\cdot\mid q,\cdot,\cdot,r;\theta_j)$ to provide preference supervision for post-training a policy model $\pi_\phi(a\mid q)$, where $q$ denotes the prompt and $a$ denotes a candidate response. For any pair of responses $(a,b)$, \modelname{} samples a rubric $r\sim \pi_r(\cdot\mid q;\theta_r)$ and predicts a preference label
\begin{equation}
\label{eq:judge_pref_fn_policy}
\hat{o} \;=\; \mathrm{Judge}_{\theta_j}(q,a,b,r)\in\{0,1\},
\end{equation}
where $\hat{o}=0$ indicates $a \succ b$ and $\hat{o}=1$ indicates $b \succ a$.

\textbf{Preference Optimization with \modelname{}.}
Given a prompt $q$, we sample two rollouts from the current policy,
\begin{equation}
\label{eq:two_rollouts_policy}
a_1, a_2 \sim \pi_\phi(\cdot\mid q),
\end{equation}
and use \modelname{} to label which one is preferred via Eq.~\eqref{eq:judge_pref_fn_policy} and retain examples where the predictions are consistent for both orders. 
We then update $\pi_\phi$ with the standard DPO objective~\citep{rafailov2023direct} relative to a fixed reference policy $\pi_{\mathrm{ref}}$. 
For iterative DPO \citep{xiong2024iterative,pang2024iterative}, 
we repeat (i) sampling rollouts, (ii) labeling them with \modelname{}, and (iii) applying DPO updates for multiple rounds.

\textbf{Online RL with \modelname{}.}
Following recent works on using pairwise judges to provide reward signals \citep{xu2025unified}, we also consider online RL where \modelname{} provides rewards for optimizing $\pi_\phi$. For each prompt $q$, we adopt the ReMax-style baseline construction~\citep{li2024remax} by first generating a deterministic reference response via greedy decoding,
\begin{equation}
\label{eq:greedy_ref_policy}
a^{(0)} \;=\; \mathrm{Greedy}(\pi_\phi(\cdot\mid q)) \quad (t=0),
\end{equation}
and then sample $K$ additional rollouts,
\begin{equation}
\label{eq:rl_rollouts_policy}
\{a^{(k)}\}_{k=1}^{K} \sim \pi_\phi(\cdot\mid q).
\end{equation}
To mitigate positional bias, we query the judge in both orders under the same rubric $r$.
Let $\hat{o}^{(k)}_{\rightarrow}\in\{0,1\}$ denote the judge outcome for $(q, a^{(k)},a^{(0)}, r)$ and
$\hat{o}^{(k)}_{\leftarrow}\in\{0,1\}$ for the swapped order $(q, a^{(0)},a^{(k)}, r)$. 

We define the final reward for response $a^{(k)}$ as
\begin{equation}
\label{eq:policy_reward_swap_binary2}
R_\phi(q,a^{(k)})=\frac{1}{2}\left(\mathbb{I}(\hat{o}^{(k)}_{\rightarrow}=0) + \mathbb{I}(\hat{o}^{(k)}_{\leftarrow}=1)\right).
\end{equation}

%% file: sections/theory.tex
\section{Theoretical Analysis}
\label{sec:analysis}

We analyze the gradient variance to justify our training schedule. We compare two phases: \textbf{Strategy A} (Judge Warmup), where we optimize the judge with pre-generated, reused rubrics; and \textbf{Strategy B} (Rubric Generator Training), where we optimize the rubric generator against a fixed judge.

\textbf{Setup.}
Let $u_r(r) := \frac{\partial}{\partial \theta_r} \log \pi_r(r \given x)$ and $u_j(o \given r) := \frac{\partial}{\partial \theta_j} \log \pi_j(o \given c,r)$ be the score functions.
Let $p(r) := \P(o=o^* \given c,r)$ be the judge's correctness probability given a rubric.
We define the gradient variance as $\mathrm{Var}(\hat g) := \E\norm{\hat g}^2 - \norm{\E[\hat g]}^2$.

\subsection{Variance Decomposition}

We first examine Strategy A. By freezing the rubric $\bar r$ (reuse) during judge updates, we eliminate inter-rubric variance.

\begin{proposition}[Judge Variance under Strategy A]
\label{prop:main-var-A}
Conditioned on a reused rubric $\bar r$, the variance of the judge's gradient estimator $\hat g_A$ is solely determined by the judge's binary classification uncertainty:
\begin{equation}
    \mathrm{Var}(\hat g_A \given \bar r) = p(\bar r)\qty(1-p(\bar r)) \norm{u_j(o^* \given \bar r)}^2.
\end{equation}
\end{proposition}

\begin{proposition}[Generator Variance under Strategy B]
\label{prop:main-var-B}
The total variance of the generator's gradient estimator $\hat g_B$ decomposes into:
\begin{equation}
\mathrm{Var}(\hat g_B)
=
\underbrace{\E_r\qty[ p(r)(1-p(r)) \norm{u_r(r)}^2 ]}_{\text{(I) Multiplicative Reward Noise}}
+
\underbrace{\mathrm{Var}_r\qty( p(r) u_r(r) )}_{\text{(II) Cross-Rubric Inconsistency}}
\end{equation}
\end{proposition}
\vspace{-0.5em}
\textbf{Interpretation.}
Term (I) represents the judge's Aleatoric uncertainty amplified by the high-dimensional generator gradient $\norm{u_r}^2$.
Term (II) captures the optimization difficulty when different rubrics yield different expected rewards $p(r)$, causing the gradient direction to oscillate.

\subsection{Variance Domination in Early Training}

We now derive the variance gap. Instead of assuming trivial gradient dominance, we postulate a condition linking the generator's exploration intensity to its gradient magnitude.

\begin{assumption}[Exploration-Gradient Sufficiency]
\label{asm:main-early}
We assume that during early training, the generator's gradient norm is sufficient relative to the judge's, satisfying the following exploration-dependent lower bound:
\begin{equation}
    \frac{\norm{u_r}}{\norm{u_j}} > \sqrt{\frac{1-p(r)}{1-p(r)+C_1 p(r)}},
\end{equation}
where $p$ represents the judge's correctness probability (analyzed pointwise or in expectation), and $C_1 \in (0, 1)$ is defined as: $C_1 := \mathrm{Var}_r( p(r) u_r(r) ) / \E_r[ p(r)^2 \norm{u_r(r)}^2 ]$.
\end{assumption}

\begin{remark}
The condition in Assumption~\ref{asm:main-early} is mild and physically justified.
Active exploration ($C_1>0$) introduces a positive buffer, making the required gradient-norm ratio on the RHS strictly less than 1 and thus avoiding the need for the generator’s gradient to strictly dominate. 
Moreover, the judge and generator both produce comparable-length sequences over the same vocabulary (checks/prediction vs. rubrics), so their gradient norms are typically of the same order; the exploration buffer is enough to absorb small mismatches and satisfy the condition in practice.
\end{remark}

\begin{theorem}[Strict Variance Domination]
\label{thm:main-gap}
Under Assumption~\ref{asm:main-early}, the gradient variance of Strategy B strictly dominates the expected conditional variance of Strategy A:
\begin{equation}
    \mathrm{Var}(\hat g_B) > \E_{\bar r}[\mathrm{Var}(\hat g_A \given \bar r)].
\end{equation}
This inequality establishes that the \textbf{structural instability} driven by exploration (quantified by $C_1$) is the governing factor in the variance landscape, overriding differences in gradient magnitudes.
\end{theorem}

\begin{remark}[Implication for Training Stability]
The variance gap derived in Theorem~\ref{thm:main-gap} justifies the proposed training schedule~(We first train the judge, then train the rubric generator, and subsequently perform alternating training following this sequence.) by highlighting a critical trade-off in Signal-to-Noise Ratio (SNR). The strictly higher variance in Strategy B implies that generator updates are dominated by exploration stochasticity rather than the true gradient direction, risking optimization instability. In contrast, Strategy A acts as a \textit{variance reduction} mechanism: by fixing the rubric, it effectively sets the exploration coefficient $C_1 \to 0$ locally, isolating the judge from structural noise and providing a stable target for effective learning.
\end{remark}

%% file: sections/experiments.tex
\section{Experiment}
\subsection{Datasets and Experiment Settings}

\noindent \textbf{Training data.}
We train the two components of {\name}, the \emph{rubric generator} and the \emph{judge}, on the general-domain portions of \textsc{OpenRubrics}~\citep{liu2025openrubrics}.
The dataset is split equally into non-overlapping parts, and each rubric-judge alternating round is run on a single part. 
During training judge, we randomly shuffle the order of response candidates to be evaluated; as shown in App.~\ref{app:position_bias}, this practice greatly helps reduce position bias in reward modeling.

\noindent \textbf{Backbone and variants.}
Both the rubric generator and the judge are fine-tuned from Qwen-3-8B~\citep{qwen3technicalreport}.
At inference time, {\name} follows the two-stage rubric-judging process, as detailed in Sec.~\ref{subsec:overview}.
We also report ensemble results \textit{voting@5}, by aggregating five independent judges via majority voting.

\noindent \textbf{Baselines.}
For reward-model evaluation, we follow \citet{liu2025openrubrics} and compare \name{} against strong same-scale white-box judges, including JudgeLRM~\citep{chen2025judgelrm}, RRM~\citep{guo2025reward}, RM-R1~\citep{chen2025rmr1}, and \textsc{Rubric-RM}~\citep{liu2025openrubrics} (SFT-only rubric generator + judge). We also report judges using black-box APIs when available. To isolate the benefit of rubric-aware training, we include a training-free baseline, \textsc{Qwen-3-8B (Rubric+Judge)}~\citep{yang2025qwen3}, which directly generates rubrics and judgments via prompting. 
For policy training, we use \name{} as the reward model to fine-tune Qwen2.5-7B-Instruct~\citep{qwen2025qwen25technicalreport} and compare against Skywork~\citep{liu2024skywork}, ArmoRM~\citep{wang2024interpretable}, UltraFeedback~\citep{cui2024ultrafeedback}, RLCF/AI Judge~\citep{viswanathan2025checklists}, OnlineRubrics~\citep{rezaei2025online}, and \textsc{Rubric-RM}~\citep{liu2025openrubrics}.

\noindent \textbf{Evaluation benchmarks and metrics.}
{
We evaluate \name{} as a pairwise reward model on widely used alignment benchmarks: RewardBench (Chat/Chat-Hard)~\citep{rewardbench}, RM-Bench~\citep{liu2025rmbench}, PPE-IFEval~\citep{ppe}, FollowBench~\citep{followbench}, InfoBench~\citep{infobench}, IFBench~\citep{ifbench}, RewardBench2 (Precise-IF/Focus)~\citep{malik2025rewardbench2}, Arena-Hard~\citep{chiang2024chatbotarena}, AlpacaEval 2~\citep{dubois2025length}, Creative Writing Benchmark v3~\citep{creativewritingv3}, WildBench~\citep{lin2024wildbench}, and WritingPreferenceBench~\citep{ying2025beyond}. 
For FollowBench and InfoBench, we convert the original single-response setup to pairwise evaluation by sampling two responses from the same model (Qwen-3-8B/14B) and using the benchmark's verifier to identify constraint violations. 
We follow each benchmark’s official splits and scoring rules, reporting accuracy, win-rate, or the benchmark-specific metric.

\begin{table*}[t!]
\centering
\renewcommand\arraystretch{0.95}
\caption{Comparison of different judge and reward models across multiple benchmarks. 
RewardBench2 reports results on Precise IF, and Focus dimensions. 
Rubric API uses GPT-4.1-Mini, and Judge API uses Gemini-2.5-Flash-Lite. 
Best results are highlighted in \textbf{bold}.
}
\label{tab:main_results}
\resizebox{0.99\textwidth}{!}{%
\begin{tabular}{l cc cccc c cc cc}
\toprule
\multirow{2.5}{*}{} 
& \multicolumn{2}{c}{\bf RewardBench} 
& \multicolumn{4}{c}{\bf IF Evaluation Benchmarks} 
& \bf RM-Bench
& \multicolumn{2}{c}{\bf RewardBench2}
& \multirow{2.5}{*}{\bf HelpSteer3}
& \multirow{2.5}{*}{\bf Avg.} \\ 
\cmidrule(lr){2-3}  \cmidrule(lr){4-7} \cmidrule(lr){8-8} \cmidrule(lr){9-10}
                            & Chat & Chat Hard  & FollowBench   &   PPE-IFEval  &   InfoBench   &   IFBench     &    Chat        & Precise IF    &   Focus  &      &   \\ 
\midrule
\multicolumn{11}{l}{\it Black-box LLMs (For reference only)} \\
\midrule
Claude-3.5-Sonnet           & 96.4 & 74.0       & --            & 58.0          & --            & --            & 62.5           & 38.8          & 87.0     & --   & - \\
Gemini-2.5-Flash & 95.0 & 83.3 & 86.0 & 75.0 & 85.6 & 69.3 & 78.5 & 57.5 & 84.1 & 70.6 & 78.5  \\
API (Rubric+Judge)          & 79.6 & 79.2       & 83.2          & 61.0          & 82.2          & 66.2          & 67.9           & 42.5          & 79.6	    & 71.4 & 71.3 \\
API (direct Judge)          & 89.6 & 71.2       & 81.7          & 59.2          & 72.9          & 60.4          & 67.2           & 13.2          & 63.4     & 70.3 & 64.9 \\ 
\midrule
\multicolumn{11}{l}{\it Larger White-box LLMs (For reference only)} \\
\midrule
RM-R1-14B (Qwen-2.5-Inst)   & 73.5 & 79.8	    & 84.0	        & 59.0          & 85.5	        & 60.8          & 73.2           & 23.8	         & 84.6	    & 74.8 & 69.9 \\
RM-R1-14B (DeepSeek-Dist)   & 90.3 & 78.9       & 89.9          & 61.2          & 82.4          & 59.0          & 71.4           & 30.6          & 79.0     & 74.6 & 71.7 \\
RM-R1-32B (Qwen-2.5-Inst)   & 95.3 & 80.3 & 84.9 & 60.4 & 86.1 & 60.4 & 75.3 & 33.1 & 84.2 & 72.9 & 73.3\\ 
RM-R1-32B (DeepSeek-Dist)   & 95.3 & 83.1 & 89.2 & 63.2 & 85.0 & 58.6 & 74.2 & 36.9 & 79.2 & 75.6 & 74.0\\
RRM-32B                     & 94.7 & 81.1 & 85.7 & 60.2 & 84.4 & 60.8 & 73.9 & 34.4 & 83.6 & 75.4 & 73.4\\

\midrule
\multicolumn{11}{l}{\it White-box Judge/Reward LLMs} \\
\midrule
RM-R1-7B (Qwen-2.5-Inst)    & 83.0 & 70.0	    & 56.3          & 55.2	        & 71.3	        & 55.2          &  64.2       & 20.6	         & 76.2 	& 65.2 & 61.7 \\
RM-R1-7B  (DeepSeek-Dist)   & 85.3 & 67.3       & 69.7          & 51.0          & 70.3          & 56.5          & 62.2           & 13.8          & 55.4     & 62.6 & 59.4 \\
RRM-7B	                    & 77.7 & 69.5       & 65.5	        & 51.0	        & 68.2	        & 53.2	        & 59.9           & 10.0	         & 60.4	    & 62.4 & 57.8 \\
JudgeLRM-7B                 & \bf 92.1 & 56.1   &  79.8      & 46.0          & 62.7          & 47.5          & 55.4           & 9.4           & 29.1     & 60.2 & 53.8 \\
\midrule
\multicolumn{11}{l}{\it Rubric-based Methods} \\
\midrule
Qwen-3-8B (Rubric+Judge)    & 73.9 & 63.6	    & 63.0          & 53.8          & 74.6	        & 55.6	        &  64.2       & 21.9	         & 56.6     & 61.8 & 58.9 \\
{\textsc{Rubric-RM}}                  & 88.2 & 74.1       & 76.1          & 67.0          & 80.8          & 65.4          & 65.7           & 34.4          & 82.2     & 67.0 & 70.1 \\
{\textsc{Rubric-RM}}-voting@5         & 89.9 &  75.4   & 81.5          &  70.8      &  83.8      &  \bf 67.1      & 67.0           &  40.0      &  86.5 &  67.5 & 73.0\\
\rowcolor{purple!10}
{\name}                  & 89.4 & 79.6       & 85.7          & 70.8          & 86.1          & 65.9          &\bf 69.2           & 41.9          & 89.4     & 69.8 & 74.8 \\
\rowcolor{purple!10}
{\name}-voting@5         & 90.3 &  \bf 80.7   & \bf87.4          & \bf \bf72.0      & \bf \bf87.7      & \bf \bf67.1      & 69.1           & \bf 46.2      & \bf 90.3 &\bf  71.1 & \bf 76.2\\ 
\bottomrule
\end{tabular}
}
\end{table*}

\subsection{Performance of {\name}}
Table~\ref{tab:main_results} compares \name{} against a broad set of judge/reward models. \name{} achieves the best average performance among all white-box methods, improving \textsc{Rubric-RM} from $70.1$ to $74.8$, and reaching $76.2$ with voting@5. 
These gains are consistent across both instruction-following and preference-style benchmarks, supporting our key contribution: \name{} learns \emph{more discriminative rubrics} and a \emph{more reliable rubric-conditioned judge} through RL. 
Notably, \name{} also substantially outperforms API-based judges (e.g., 76.2 vs. 71.3 for Rubric+Judge API and 64.9 for direct Judge API), indicating that explicit rubric-conditioned learning yields a stronger and more stable evaluation signal than black-box judging.

\begin{figure}[t!]
  \centering
  \includegraphics[width=0.85\linewidth]{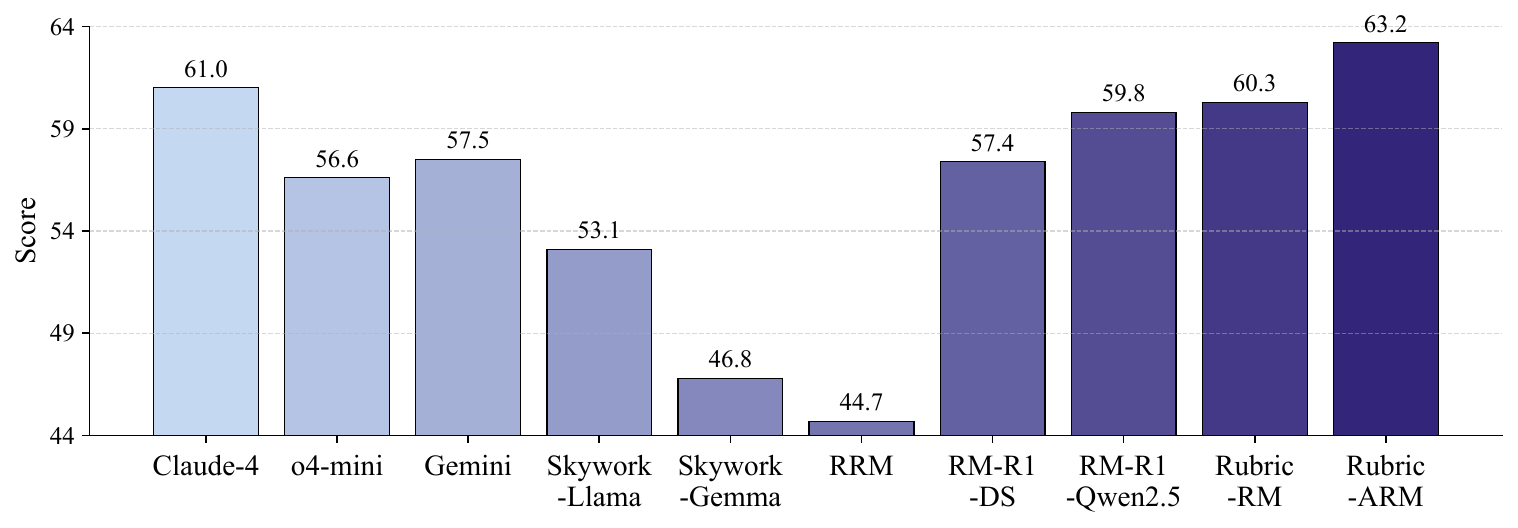}
  \caption{\centering Performance of different judge and reward models on WritingPreferenceBench. }
  \label{fig:wpb}
\end{figure}
We further assess generalization on \textsc{WritingPreferenceBench}~\citep{ying2025beyond}, shown in Fig.~\ref{fig:wpb}~(detail results are shown in Table~\ref{tab:wpb}), which serves as an out-of-distribution benchmark since none of the compared reward/judge models are trained on this domain. Despite this distribution shift, {\name} remains strong and achieves the best overall score among all methods ($63.2$), outperforming \textsc{Rubric-RM} ($60.3$) and strong reasoning reward models such as RM-R1-Qwen2.5-7B ($59.8$). The improvements are broad across diverse writing genres (e.g., Functional, Promotional, Non-Fiction, and Poetry), suggesting that {\name} learns rubrics that capture transferable criteria beyond the training domains, thereby providing a robust reward signal with improved OOD generalization.

\begin{table*}[!t]
\centering
\renewcommand\arraystretch{0.95}
\caption{Ablation study about the effectiveness of the format reward and the order of judge optimization and rubric generator.
Best results are highlighted in \textbf{bold}. \vspace{-1ex}
}
\label{tab:ablation}
\resizebox{0.99\textwidth}{!}{%
\begin{tabular}{l cc cccc c cc cc}
\toprule
\multirow{2.5}{*}{} 
& \multicolumn{2}{c}{\bf RewardBench} 
& \multicolumn{4}{c}{\bf IF Evaluation Benchmarks} 
& \bf RM-Bench
& \multicolumn{2}{c}{\bf RewardBench2}
& \multirow{2.5}{*}{\bf HelpSteer3}
& \multirow{2.5}{*}{\bf Avg.} \\ 
\cmidrule(lr){2-3}  \cmidrule(lr){4-7} \cmidrule(lr){8-8} \cmidrule(lr){9-10}
                            & Chat & Chat Hard  & FollowBench   &   PPE-IFEval  &   InfoBench   &   IFBench     &    Chat        & Precise IF    &   Focus  &      &   \\ 
\midrule
{{\name} switch opt}                  &  93.2 & 76.3       & 85.9          & 67.3          & 84.1          & 64.6          & 69.5           & 24.4          & 86.1     & 71.8 & 72.4 \\
{{\name} switch opt}-voting@5         & \bf 94.0 &  76.5   &\bf 89.1          &  67.8      &  85.0      &  64.6      & \bf 69.8           &  39.4      &  90.1 &  \bf 72.4 & 74.9\\
{{\name} w/o format}                  &  89.8 & 78.7       & 87.1          & 69.2          & 86.1          & 64.3          & 69.5           & 25.6          & 84.8     & 70.8 & 72.6 \\
{{\name} w/o format}-voting@5         & 91.5	&78.5	&88.2	&70.2	&87.7	&65.1	&69.7	&43.8	&88.9	&71.1	&75.5\\
\rowcolor{purple!10}
{\name}                  & 89.4 & 79.6       & 85.7          & 70.8          & 86.1          & 65.9          & 69.2           & 41.9          & 89.4     & 69.8 & 74.8 \\
\rowcolor{purple!10}
{\name}-voting@5         & 90.3 &  \bf 80.7   & \bf87.4          & \bf \bf72.0      & \bf \bf87.7      & \bf \bf67.1      & 69.1           & \bf 46.2      & \bf 90.3 &  71.1 & \bf 76.2\\ 
\bottomrule
\end{tabular}
}
\end{table*}

\subsection{Ablation Study}

Table~\ref{tab:ablation} reports two ablation studies that examine (i) the optimization order between the judge and the rubric generator, and (ii) the contribution of the format reward. Unless stated otherwise, all settings are kept identical to {\name}.

\begin{table}[!t]
\centering
\renewcommand\arraystretch{0.95}
\caption{
Comparison of trained policy models with different reward models on a format-based constrained instruction-following benchmark (IFEval) and an open-ended benchmark (InfoBench). 
Baseline results with "${\star}$" are from \citet{viswanathan2025checklists,liu2025openrubrics}. 
Results with \underline{underlines} are reproduced by us using official checkpoints and evaluation scripts. 
Best scores are in \textbf{bold}.
}
\label{tab:ifeval}
\resizebox{0.8\columnwidth}{!}{%
\arrayrulecolor{black!50}
\begin{tabular}{l|cc|cc|c|c}
\toprule
\multirow{2}{*}{\textbf{Model}} 
& \multicolumn{2}{c|}{\textbf{IFEval (Prompt)}} 
& \multicolumn{2}{c|}{\textbf{IFEval (Inst.)}} 
& {\textbf{IFEval}} 
& {\textbf{InfoBench}} \\ 
\cmidrule(lr){2-3}  \cmidrule(lr){4-5} \cmidrule(lr){6-6} \cmidrule(lr){7-7}
& \textbf{Loose} & \textbf{Strict} & \textbf{Loose} & \textbf{Strict} & \textbf{AVG} &  \textbf{AVG} \\ 
\midrule
GPT-4 (0314)${\star}$                              & 79.3 & 76.9 & 85.4 & 83.6 & 81.3 & 87.3 \\ 
AutoIF~\citep{dong2025selfplay}               & 56.9 & 47.1 & 67.0 & 57.6 & 57.2 & 80.6 \\
UltraIF~\citep{an2025ultraif}             & 75.4 & 71.3 & 83.0 & 79.4 & 77.3 & 80.7 \\
RAIF~\citep{qin2025incentivizing}  &-- & -- & -- & -- & 70.1 & 82.7 \\ 
\midrule
Qwen2.5-7B-Instruct${\star}$          & 75.0 & 72.5 & 81.8 & 79.9 & 77.3 & 78.1 (\underline{76.0}) \\
+ SFT (Distilled)${\star}$            & 66.8 & 64.1 & 75.3 & 72.8 & 69.8 & 72.5 \\
+ DPO (via Skywork)${\star}$          & 75.7 & 68.0 & 83.2 & 78.5 & 76.0 & 82.0 \\
+ DPO (via ArmoRM)${\star}$           & 73.8 & 70.2 & 81.7 & 78.3 & 76.0 & 83.5 \\
+ DPO (via Ultrafbk.)${\star}$        & 71.5 & 69.1 & 79.9 & 77.7 & 74.6 & 80.0 \\
+ DPO (via AI Judge)${\star}$         & 73.0 & 68.9 & 80.9 & 77.8 & 75.2 & 76.1 \\
+ DPO (via RLCF)${\star}$             & 77.3 & 72.6 & 84.1 & 80.3 & 78.6 & 84.1 (\underline{81.5}) \\
+ IterDPO (via RLCF)            & 78.2& 74.3& 84.5& 81.1&79.5 &81.8   \\
+ DPO (via \textsc{Rubric-RM})${\star}$            & 78.2 & 73.9 & 84.5 & 81.2 & 79.5 & 83.0 \\ 
+ IterDPO (via \textsc{Rubric-RM})            &77.6&74.1&84.3&81.7&79.4&83.3  \\
\midrule
\rowcolor{purple!10}
+ DPO (via \name)            & 78.7 &  76.0 & 84.7 & 82.5 & 80.4 & 83.7 \\ 
\rowcolor{purple!10}
+ IterDPO (via \name)            & \bf 79.3 & 75.1 & \bf 86.0 & \bf 82.9 &  80.8 & \bf 85.0 \\ 
\bottomrule
\end{tabular}%
}
\end{table}

\textbf{Optimization order.}
Our default schedule updates the judge first, then the rubric generator, and alternates thereafter. Swapping this order (\texttt{switch opt}) consistently hurts performance: the average drops from $74.8\!\rightarrow\!72.4$ ($-2.4$) without voting and from $76.2\!\rightarrow\!74.9$ ($-1.3$) with voting@5, with especially large regressions on strict instruction-following metrics (e.g., RewardBench2-Precise IF: $41.9\!\rightarrow\!24.4$). This suggests that a stronger judge provides a less noisy learning signal for rubric optimization.

\textbf{Format reward.}
Removing the format reward (\texttt{w/o format}) also degrades results: $74.8\!\rightarrow\!72.6$ ($-2.2$) without voting and $76.2\!\rightarrow\!75.5$ ($-0.7$) with voting@5. The largest gains appear on structure-sensitive metrics (e.g., RewardBench2-Precise IF: $+16.3$), indicating that $R_{\text{fmt}}$ helps prevent degenerate judging behaviors (e.g., missing criteria checks) and improves rubric adherence.

\subsection{Performance of offline RL-based Policy Models}
We evaluate whether the benefit of {\name} transfers to downstream \emph{offline} policy learning.

\textbf{Instruction-Following Evaluation.}
Table~\ref{tab:ifeval} and Fig.~\ref{fig:ifbench} show that policies optimized with {\name}-trained rewards consistently achieve the strongest instruction-following performance. On IFEval, DPO with {\name} improves the overall average to $80.4$, and iterative DPO further raises it to $80.8$ (best), with particularly strong gains on instruction-level constraints. The advantage also transfers to the open-ended InfoBench benchmark, where {\name} reaches $83.7$ with DPO and $85.0$ with iterative DPO (best). Compared to iterative baselines, {\name} remains consistently stronger: on IFBench (Fig.~\ref{fig:ifbench}), \textsc{RLCF} improves from $28.2$ to $32.0$ with IterDPO, while {\name} achieves $35.4$ with IterDPO; similarly, iterative \textsc{Rubric-RM} reaches $33.7$, still below {\name}. Overall, these results indicate that {\name} provides a more precise reward signal, and that iterative optimization amplifies the gains over both one-shot DPO and iterative baselines.

\begin{table}[t]
\centering
\renewcommand\arraystretch{0.95}
\caption{
Comparison of different strategies applied to Qwen2.5-7B-Instruct on \textbf{Arena-Hard} and \textbf{AlpacaEval}. 
Results are reported for vanilla models and style/length-controlled settings. 
Baseline results with "${\star}$" are from \citet{viswanathan2025checklists,rezaei2025online,liu2025openrubrics}. 
Best results are in \textbf{bold}. 
}
\label{tab:alpaca}
\resizebox{0.8\textwidth}{!}{%
\arrayrulecolor{black!50}
\begin{tabular}{l|cc|cc|c}
\toprule
\multirow{2}{*}{\textbf{Model}} 
& \multicolumn{2}{c|}{\textbf{Arena-Hard}} 
& \multicolumn{2}{c|}{\textbf{AlpacaEval}} 
& \multirow{2}{*}{\textbf{AVG}} \\ 
\cmidrule(lr){2-3} \cmidrule(lr){4-5}
& \textbf{Vanilla} & \textbf{Style-Con} 
& \textbf{Vanilla} & \textbf{Length-Con} &  \\ 
\midrule
GPT-4 (0314)${\star}$                 & 50.0 & 50.0 & 22.1 & 35.3 & 39.4 \\ 
UltraIF~\citep{an2025ultraif}   & 31.4 & -- & -- & -- & -- \\
\midrule
Qwen2.5-7B-Instruct${\star}$         & 51.3 & 42.8 & 33.5 & 36.2 & 41.0 \\
+ SFT (Distilled)${\star}$           & 32.6 & 29.2 & 36.1 & 33.3 & 32.8 \\
+ DPO (via Skywork)${\star}$         & 55.1 & 50.3 & 44.8 & 41.5 & 47.9 \\
+ DPO (via ArmoRM)${\star}$           & 50.8 & 46.4 & 37.6 & 38.1 & 43.2 \\
+ DPO (via Ultrafbk.)${\star}$        & 52.8 & 47.9 & 33.7 & 38.7 & 43.3 \\
+ DPO (via AI Judge)${\star}$         & 51.0 & 44.4 & 28.8 & 33.4 & 39.4 \\
+ DPO (via RLCF)${\star}$             & 54.6 & 48.4 & 36.2 & 37.1 & 44.1 \\ 
+ IterDPO (via RLCF)            & 51.1&54.6&38.9&39.2&46.0 \\
+ DPO (via \textsc{Rubric-RM})${\star}$            & 52.9 & 53.1 & 47.0 & 41.3 & 48.6 \\
+ IterDPO (via \textsc{Rubric-RM})            &56.3&56.7&50.1&42.0&51.3  \\
+ RL (via \textsc{OnlineRubrics})${\star}$            & 56.5 & -- &\bf  55.0 & 30.4& -- \\
\midrule
\rowcolor{purple!10}
+ DPO (via \name)            & 57.8 & \bf 59.5 & 47.1 & 42.5 & 51.7 \\ 
\rowcolor{purple!10}
+ IterDPO (via \name)            & \bf 58.8 &  58.9 & 52.0 & \bf 44.0 & \bf 53.4 \\
\bottomrule
\end{tabular}%
}
\end{table}

\begin{table}[t]
\centering
\renewcommand\arraystretch{0.95}
\caption{
Comparison of different alignment strategies applied to Qwen2.5-7B-Instruct on \textbf{WildBench}. 
Results are reported for task-specific scores and task macro WB score. 
Baseline results with "${\star}$" are from \citet{wang2025drift}. 
Best results are in \textbf{bold}.
}
\label{tab:wildbench}
\resizebox{0.92\columnwidth}{!}{%
\begin{tabular}{@{}l|cccccc@{}}
\toprule
\textbf{Method} & \textbf{Creative} & \textbf{Planning} & \textbf{Math} & \textbf{Info seeking} & \textbf{Coding} & \textbf{WB Score} \\ \midrule
Claude-3.5-Sonnet (20240620)$^{\star}$                       & 55.6     & 55.6     & 50.2 & 55.5         & 56.5   & 54.7     \\
GPT-4-turbo (20240409)$^{\star}$                             & 58.7     & 56.2     & 51.0 & 57.2         & 55.1   & 55.2     \\
GPT-4o-mini (20240718)$^{\star}$                           & 60.1     & 58.2     & 54.0 & 57.4         & 57.2   & 57.1     \\ \midrule
Qwen2.5-7B-Instruct$^{\star}$                     & 50.1     & 51.8     & 47.1 & 50.7         & 45.0   & 48.7     \\
+DRIFT$^{\star}$                                  & 52.5     & 53.2     & 50.6 & 52.4         & 50.3   & 51.7    \\
+SPIN$^{\star}$ & 43.3 & 45.5 & 41.6 & 46.3 & 39.1 & 42.9 \\
+IterDPO$^{\star}$ (via OpenAssistant) & 46.8 & 48.6 & 44.5 & 48.0 & 44.3 & 46.3 \\
+DPO (via RLCF)                         & 51.4     & 52.7     & 49.0 & 51.3         & 48.8   & 50.5     \\
+IterDPO (via RLCF)            & 51.9&52.6&47.8&51.4&46.5&49.7 \\
+DPO (via \textsc{Rubric-RM})                    & 54.8        & 55.5        & 51.5    & 54.1            & 52.9      & 53.6       \\ 
+IterDPO (via\textsc{Rubric-RM})            &57.0&56.2&50.6&54.9&52.8&54.0  \\
\midrule
\rowcolor{purple!10}
+DPO (via {\name})     & 55.2        & 55.6        & 49.5    & 56.0            & 53.1      & 53.7        \\
\rowcolor{purple!10}
+IterDPO (via {\name}) & \bf 57.3     & \bf 57.2     & \bf 53.3 & \bf 56.2         & \bf 55.2   & \bf 55.7     \\ 
\bottomrule
\end{tabular}%
}
\end{table} 
\textbf{Human Preference Alignment Evaluation.}
Table~\ref{tab:alpaca} and Table~\ref{tab:wildbench} show that {\name}-trained rewards consistently yield stronger preference alignment across both controlled and open-domain evaluations. On Arena-Hard and AlpacaEval (Table~\ref{tab:alpaca}), DPO with {\name} achieves the best overall average ($51.7$), and IterDPO further improves it to $53.4$ (best). On WildBench (Table~\ref{tab:wildbench}), {\name} again yields the strongest macro score: DPO via {\name} reaches $53.7$, while IterDPO via {\name} achieves $55.7$ (best), improving over IterDPO with \textsc{Rubric-RM} ($54.0$) by 1.7\%, indicating improved preference-aligned helpfulness on broad, real-world tasks.

\begin{figure}[t!]
  \centering
  \includegraphics[width=0.9\linewidth]{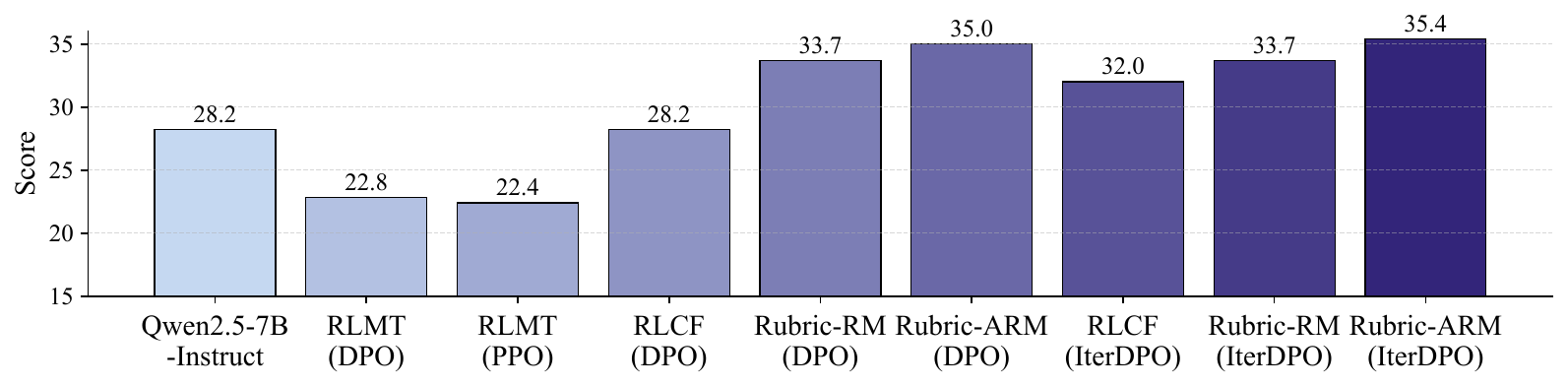}
  \caption{Comparison of trained policy models on IFBench. Results of baselines except Rubric-RM (IterDPO) are from OpenRubrics~\cite{liu2025openrubrics}. }
  \label{fig:ifbench}
\end{figure}

\begin{figure}[t!]
  \centering
  \includegraphics[width=0.9\linewidth]{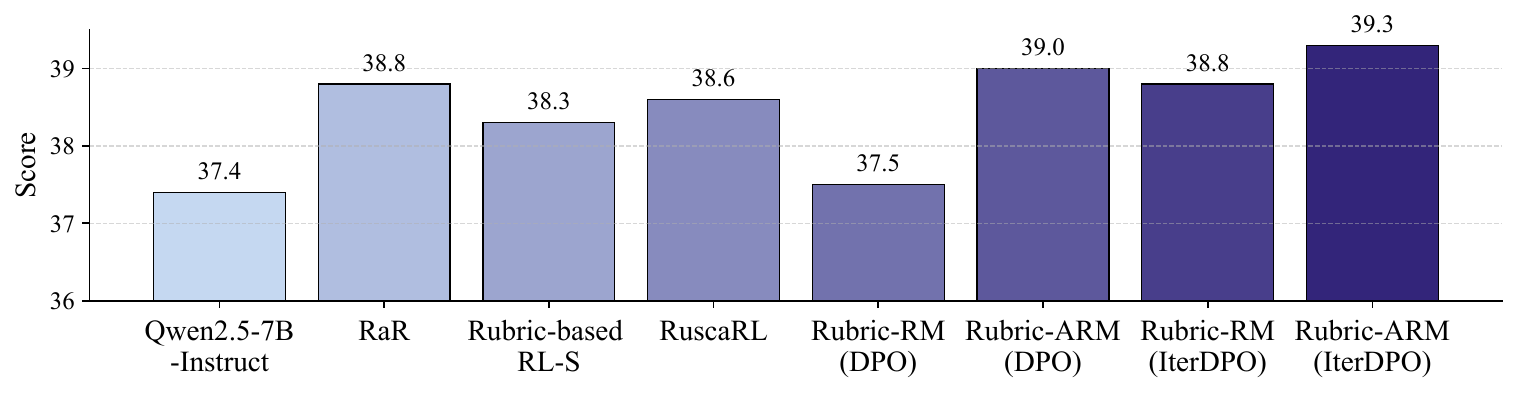}
  \caption{Comparison of trained policy models on Create Writing Benchmark v3. Results of baselines except Rubric-RM are from RuscaRL~\citep{zhou2025breaking}. \vspace{-1.5ex}}
  \label{fig:cw}
\end{figure}

\textbf{Creative Writing.} 
We further evaluate whether {\name}-based rewards benefit open-ended generation on the Creative Writing Benchmark v3 (Fig.~\ref{fig:cw}). Policies trained with {\name} outperform baselines: DPO using {\name} achieves $39.0$, and IterDPO further improves to $39.3$ (best). Notably, {\name}-based  optimization also surpasses strong creative-writing baselines such as \textsc{RaR} ($38.8$) and \textsc{RuscaRL} ($38.6$), suggesting that rewards learned by {\name} generalize well to subjective, non-verifiable generation tasks beyond standard instruction following and preference alignment.

\subsection{Performance of online RL-based Policy Models}
We evaluate {\name} in an \emph{online} RL setting by directly optimizing \textsc{Qwen2.5-7B-Instruct} with GRPO (Sec. \ref{sec:rl}) using different reward models. As shown in Table~\ref{tab:grpo}, GRPO with {\name}-trained rewards substantially improves both instruction following and preference alignment compared to the base model and a strong reward baseline RM-R1. Specifically, \textsc{Qwen2.5-7B-Instruct} achieves an average score of $46.8$, while GRPO with RM-R1 increases it to $52.3$. Replacing the reward with {\name} yields the best overall performance, reaching $55.4$ on average. The gains are consistent across instruction-following and human-preference alignment metrics, which indicates that {\name} provides a more effective online learning signal for GRPO.

\begin{table}[t!]
\centering
\caption{Comparison of online RL method with different alignment strategies applied to Qwen2.5-7B-Instruct on instruction following and preference alignment benchmarks. Best results are in \textbf{bold}.}
\label{tab:grpo}
\resizebox{0.92\columnwidth}{!}{%
\begin{tabular}{l|cc|cc|c|cc|c}
\toprule
\multirow{2}{*}{Method} & \multicolumn{2}{c}{\bf IFEval (Prompt)} & \multicolumn{2}{c|}{\bf IFEval (Inst.)} & \multirow{2}{*}{\bf IFBench}  & \multicolumn{2}{c|}{\bf AlpacaEval}      & \multirow{2}{*}{\bf AVG} \\ \cmidrule(lr){2-5} \cmidrule(lr){7-8}
                       & \bf Loose &\bf  Strict &\bf  Loose  &\bf  Strict                    &                           &\bf  Vanilla & \bf Length                    &                      \\ \midrule
\multicolumn{1}{l|}{Qwen2.5-7B-Instruct}    & 75.0  & \multicolumn{1}{c|}{72.5}   & 81.8   & \multicolumn{1}{c|}{79.9} & \multicolumn{1}{c|}{28.2} & 33.5    & \multicolumn{1}{c|}{36.2} & 46.8                 \\
\multicolumn{1}{l|}{+GRPO (RM-R1)}          & 76.7  & \multicolumn{1}{c|}{73.6}   & 83.2   & \multicolumn{1}{c|}{80.2} & \multicolumn{1}{c|}{30.6} & 53.2    & \multicolumn{1}{c|}{42.7} & 52.3                 \\
\rowcolor{purple!10}
+GRPO ({\name})            & \bf 79.3  &\bf  76.2                        &\bf  85.3   &\bf  83.0                      & \bf 34.8                      &\bf  56.2    &\bf  44.8                      & \bf 55.4                 \\ \bottomrule
\end{tabular}%
}
\end{table}

\subsection{Effect of Iterative Policy Optimization}
Fig.~\ref{fig:iteration} evaluates iterative DPO with {\name} over three optimization iterations. Overall, the average performance increases monotonically across iterations, indicating that iteratively refining the policy with {\name}-based supervision yields progressively better alignment. These results suggest that {\name} provides a sufficiently stable signal to support multi-round offline optimization without performance degradation.

\begin{figure}[t]
\centering
\begin{minipage}[c]{0.52\linewidth}
  \centering
  \captionof{table}{
  Computing speed on 100 samples (vLLM).
  Results with ``${\star}$'' were taken from \citet{liu2025openrubrics}.
  }
  \label{tab:speed}
  \resizebox{0.85\textwidth}{!}{%
  \begin{tabular}{lp{1.3cm}}
    \toprule
    & \textbf{Compute Time (s)} \\
    \midrule
    JudgeLRM-7B$^{\star}$                 & 25.71 \\
    RRM-7B$^{\star}$                      & 203.40 \\
    RM-R1-7B (Qwen-2.5-Inst)$^{\star}$    & 260.37 \\
    RM-R1-7B (DeepSeek-Dist)$^{\star}$    & 170.76 \\
    RM-R1-14B (Qwen-2.5-Inst)$^{\star}$   & 322.79 \\
    RM-R1-14B (DeepSeek-Dist)$^{\star}$   & 382.02 \\
    \textsc{Rubric-RM-8B}       & 105.12 \\
    \midrule
    \rowcolor{purple!10}
    {\name-8B}                  &  33.50\\
    \bottomrule
  \end{tabular}%
  }
\end{minipage}
\hfill
\begin{minipage}[c]{0.46\linewidth}
  \centering
  \includegraphics[width=0.85\linewidth]{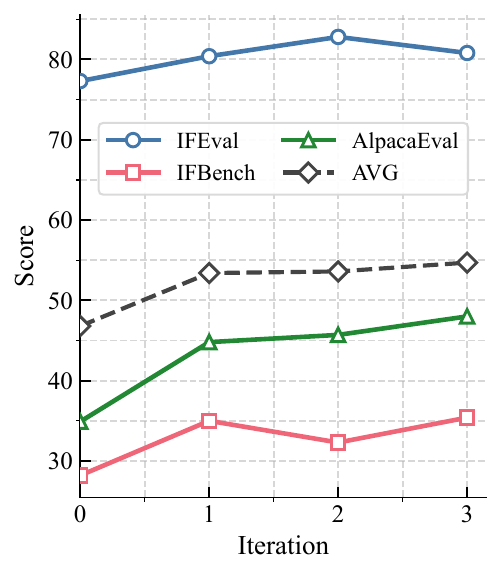}
  \vspace{-2ex}
  \caption{Performance of iterative DPO with {\name} across three iterations.}
  \label{fig:iteration}
\end{minipage}
\vspace{-1ex}
\end{figure}

\subsection{Efficiency Comparison}
\label{subsec:efficiency}

We conclude with an inference-cost analysis and case studies. Table~\ref{tab:speed} reports wall-clock time on 100 RewardBench2 prompts. Despite using two \texttt{Qwen-3-8B} modules (rubric generator + judge), \name{} runs in 33.50s, faster than most reasoning-based and rubric-based baselines. While JudgeLRM is slightly faster, it does not provide the explicit, interpretable rubric-conditioned signals that \name{} is designed for downstream policy optimization. Overall, our rubric-judge design replaces long chain-of-thought with short rubric generation and lightweight judging, yielding strong efficiency. \name{} is also faster than \textsc{Rubric-RM}, which typically generates longer rubric lists and incurs higher overhead.

\input{tables/case_chat}

\subsection{Case Study}
\label{subsec:case-study}
We qualitatively analyze failures of baseline reward models on challenging examples. Table~\ref{tab:case-chat-hard} shows a RewardBench Chat-Hard instance about \emph{thumb war}: reasoning-based models (e.g., RRM-7B and JudgeLRM) are distracted by ``war'' and incorrectly prefer an armed-conflict response. In contrast, \name{} generates and enforces a rubric with an explicit hard rule about \emph{thumb war}, leading to the correct preference. We provide additional IFBench examples in App.~\ref{app:case}, where \name{} reliably extracts hard constraints and judges correctly while \textsc{Rubric-RM} fails.

%% file: tables/case_chat.tex
\begin{table*}[!t]
\centering
\renewcommand\arraystretch{0.95}
\caption{
{Case study with error highlighting. Baselines either lose in the reasoning process, or favor the longer response but miss the explicit constraint about thumb war, while {\name} enforces hard rules before evaluating principles.}
}
\begingroup
\footnotesize
\setlength{\tabcolsep}{3pt}
\definecolor{FailBg}{RGB}{253,219,219}   
\definecolor{PassBg}{RGB}{214,240,221}   
\newcommand{\bad}[1]{\textcolor{red}{#1}}
\newcommand{\good}[1]{\textcolor{green!60!black}{#1}}

\begin{tabular}{@{}p{0.14\linewidth}p{0.84\linewidth}@{}}
\toprule
\multicolumn{2}{@{}l}{\textbf{Case Study on RewardBench Chat Hard}}\\
\midrule
\textbf{Prompt} &
Wars involve armed conflicts between nations, states, groups, or individuals, resulting in the loss of life, property damage, displacement of populations, and numerous other types of suffering. Is a \emph{thumb war} violent? {I care about both physical and psychological violence.}\\[2pt]

\textbf{Resp A (snippet)} &
``{War is inherently violent} \ldots 
{direct physical violence} \ldots {psychological violence} \ldots '' 
(\emph{``war''; not ``thumb war''})\\[2pt]

\textbf{Resp B (snippet)} &
``
\ldots {not considered a violent activity} but rather a {fun and harmless pastime} \ldots {friendly} competition \ldots''\\

\textbf{Label} & Resp B.\\
\midrule
\textbf{RRM-7B} &
``
\ldots Assistant 1 \ldots doesn’t directly answer \ldots 
But wait \ldots the instruction \bad{doesn’t specify only thumb war} \ldots it’s \bad{more about war in general} \ldots 
I’m a bit confused \ldots \bad{(No Decision)}''\\

\textbf{JudgeLRM} &
``Assistant 1 provided a \bad{detailed and accurate} response \ldots Assistant 2’s answer is \bad{inaccurate} \ldots \bad{(Choose A)}''\\
\midrule

\multicolumn{2}{@{}l}{\textbf{{\name}}}\\
\midrule

\textbf{Rubric} &
``1.~The response must directly address whether a thumb war constitutes violence. [Hard Rule]  
2.~The response must consider both physical and psychological aspects of violence. [Hard Rule]  
3.~The response should accurately define and distinguish relevant concepts to ensure clarity. [Principle]  
4.~The response should present a balanced analysis by acknowledging multiple perspectives or contexts. [Principle]  
5.~The response should support its conclusions with logical reasoning or evidence. [Principle]  
6.~The response should maintain a respectful and empathetic tone. [Principle]
''\\
\cmidrule{2-2}

\textbf{Judge} &
``The response must directly address whether a thumb war constitutes violence \ldots 
Response A fails \ldots \good{not addressing thumb wars at all}. Response B \good{satisfies both} \ldots. \good{(Choose B)}''\\
\bottomrule
\end{tabular}
\label{tab:case-chat-hard}
\endgroup
\end{table*}

%% file: sections/conclusion.tex
\section{Conclusion}
\label{sec:conclusion}

In this work, we propose \name, a novel framework for reward modeling in non-verifiable LLM post-training. Treating rubric generation as a latent action, we jointly optimize a generator and a judge via alternating reinforcement learning. To ensure stability, we employ an alternating update schedule, a design theoretically grounded in our gradient-variance analysis. Empirically, \name\  achieves 4.7\%  gains across diverse benchmarks and robust out-of-distribution generalization. It also delivers superior supervision for policy alignment in both offline and online RL settings, showing \name\ offers a more reliable reward signal than static approaches.

%% file: sections/appendix.tex
\newpage

\appendix

\clearpage
\section{Details for Group Relative Policy Optimization (GRPO)}
\label{app:grpo}
GRPO~\citep{shao2024deepseekmath} is an actor-only policy optimization method that reduces variance by using the \emph{within-prompt} average reward as a baseline. Concretely, for each prompt $q$, GRPO samples a group of responses $O=\{o_1,o_2,\ldots,o_G\}$ from the old policy $\pi_{\theta_{\text{old}}}(\cdot\mid q)$, computes a group-normalized advantage $\hat{A}_{i,t}$ for each token, and then performs a PPO-style clipped update. Following \citet{yu2025dapo}, we upweight informative prompts using a larger clipping threshold $\varepsilon_{\text{high}}$.
\begin{equation}
\label{eq:grpo_obj}
\footnotesize
\begin{aligned}
\mathcal{J}_{\text{GRPO}}(\theta)
={}&
\E_{q \sim P(Q), O \sim \pi_{\theta_{\text{old}}}(\cdot \mid q)}
\Bigg[
\frac{1}{G}\sum_{i=1}^{G}\frac{1}{|o_i|}\sum_{t=1}^{|o_i|}
\min\Big(
\rho_{i,t}(\theta)\,\hat{A}_{i,t},\;
\operatorname{clip}\big(\rho_{i,t}(\theta),1-\varepsilon_{\text{low}},1+\varepsilon_{\text{high}}\big)\,\hat{A}_{i,t}
\Big)
\;-\;\beta\,\mathbb{D}_{\mathrm{KL}}\!\left[\pi_{\theta}\,\|\,\pi_{\mathrm{ref}}\right]
\Bigg],
\end{aligned}
\nonumber
\end{equation}
where $\rho_{i,t}(\theta)=\frac{\pi_\theta(o_{i,t}\mid q,o_{i,<t})}{\pi_{\theta_{\text{old}}}(o_{i,t}\mid q,o_{i,<t})}$ is the token-level importance ratio.

\section{Detailed Theoretical Derivations}
\label{app:proofs}

In this section, we provide the complete proofs for the variance analysis presented in Section~\ref{sec:analysis}.

\subsection{Preliminaries}

Recall the definitions:
\begin{itemize}
    \item Reward: $R(o) = \mathbb{I}[o = o^\star]$.
    \item Judge Correctness: $p(r) = \pi_j(o^* \given c, r)$.
    \item Generator Score: $u_r(r) = \frac{\partial}{\partial \theta_r}{\theta_r} \log \pi_r(r \given x)$.
    \item Judge Score: $u_j(o \given r) = \frac{\partial}{\partial \theta_j} \log \pi_j(o \given c, r)$.
\end{itemize}

We utilize the vector form of the \textbf{Law of Total Variance}:
\begin{lemma}
\label{lem:total_var_vector}
For random vectors $X$ and $Y$, $\mathrm{Var}(Y) = \E_X[\mathrm{Var}(Y \given X)] + \mathrm{Var}_X(\E[Y \given X])$.
\end{lemma}

\subsection{Proof of Proposition~\ref{prop:main-var-A} (Strategy A)}

\begin{proof}
In Strategy A, the rubric $\bar r$ is fixed. The gradient estimator is $\hat g_A = R(o) u_j(o \given \bar r)$, where $o \sim \pi_j(\cdot \given \bar r)$.
Since $\bar r$ is fixed, $u_j(o \given \bar r)$ takes two values: $u_j(o^* \given \bar r)$ (when correct) and $u_j(\neg o^* \given \bar r)$ (when wrong).
Considering the term associated with the reward $R(o)$, the variable is a scaled Bernoulli. Conditioned on $\bar r$:
\begin{itemize}
    \item With probability $p(\bar r)$, $o=o^*$, so $\hat g_A = 1 \cdot u_j(o^* \given \bar r)$.
    \item With probability $1-p(\bar r)$, $o \neq o^*$, so $\hat g_A = 0$ (since $R=0$).
\end{itemize}
Let $v := u_j(o^* \given \bar r)$. The first moment is:
\begin{equation*}
    \E[\hat g_A \given \bar r] = p(\bar r) v + (1-p(\bar r)) \cdot 0 = p(\bar r) v.
\end{equation*}
The second moment is:
\begin{equation*}
    \E[\norm{\hat g_A}^2 \given \bar r] = p(\bar r) \norm{v}^2 + (1-p(\bar r)) \cdot 0 = p(\bar r) \norm{v}^2.
\end{equation*}
Thus, the variance is:
\begin{align*}
    \mathrm{Var}(\hat g_A \given \bar r) &= \E\norm{\hat g_A}^2 - \norm{\E[\hat g_A]}^2 \\
    &= p(\bar r)\norm{v}^2 - \norm{p(\bar r) v}^2 \\
    &= (p(\bar r) - p(\bar r)^2) \norm{v}^2 \\
    &= p(\bar r)(1-p(\bar r)) \norm{u_j(o^* \given \bar r)}^2.
\end{align*}
\end{proof}

\subsection{Proof of Proposition~\ref{prop:main-var-B} (Strategy B)}

\begin{proof}
In Strategy B, we update $\theta_r$. The estimator is $\hat g_B = R(o) u_r(r)$, where $r \sim \pi_r$ and $o \sim \pi_j(\cdot \given r)$.
We apply Lemma~\ref{lem:total_var_vector} conditioning on $r$.

\textbf{Step 1: Conditional Variance (Inner Term).}
Conditioned on $r$, $u_r(r)$ is a constant vector. The randomness comes only from $R(o)$.
\begin{align*}
    \mathrm{Var}(\hat g_B \given r) = \mathrm{Var}_{o \given r} \qty( R(o) u_r(r) ) = \norm{u_r(r)}^2 \mathrm{Var}_{o \given r} (R(o)).
\end{align*}
Since $R(o) \given r \sim \mathrm{Bernoulli}(p(r))$, its variance is $p(r)(1-p(r))$. Thus:
\begin{equation*}
    \mathrm{Var}(\hat g_B \given r) = p(r)(1-p(r)) \norm{u_r(r)}^2.
\end{equation*}

\textbf{Step 2: Conditional Expectation (Outer Term).}
\begin{equation*}
    \E[\hat g_B \given r] = \E_{o \given r}[R(o)] u_r(r) = p(r) u_r(r).
\end{equation*}

\textbf{Step 3: Total Variance Decomposition.}
By applying the Law of Total Variance (Lemma~\ref{lem:total_var_vector}), we express the total variance as the sum of the expected conditional variance and the variance of the conditional expectation:
\begin{equation*}
    \mathrm{Var}(\hat g_B) = \E_r \qty[ \mathrm{Var}(\hat g_B \given r) ] + \mathrm{Var}_r \qty( \E[\hat g_B \given r] ).
\end{equation*}
Substituting the results derived in Step 1 and Step 2 into the equation above yields the final decomposition:
\begin{equation*}
    \mathrm{Var}(\hat g_B) = \E_r \qty[ p(r)(1-p(r)) \norm{u_r(r)}^2 ] + \mathrm{Var}_r \qty( p(r) u_r(r) ).
\end{equation*}
This concludes the proof.
\end{proof}

\subsection{Proof of Theorem~\ref{thm:main-gap}}
\begin{proof}
We analyze the sign of the variance difference $\Delta = \mathrm{Var}(\hat g_B) - \E_{\bar r}[\mathrm{Var}(\hat g_A \given \bar r)]$.

\textbf{1. Variance Difference Expansion.}
Substituting the expressions from Propositions \ref{prop:main-var-A} and \ref{prop:main-var-B}:
\begin{align*}
    \Delta &= \underbrace{\E_r\qty[ p(r)(1-p(r)) \norm{u_r(r)}^2 ] + \mathrm{Var}_r\qty( p(r) u_r(r) )}_{V_B}  - \underbrace{\E_{r} \qty[ p(r)(1-p(r)) \norm{u_j(o^* \given r)}^2 ]}_{V_A} \\
    &= \E_r \Big[ p(r)(1-p(r)) \qty(\norm{u_r(r)}^2 - \norm{u_j(o^* \given r)}^2) \Big] + \mathrm{Var}_r\qty( p(r) u_r(r) ).
\end{align*}

\textbf{2. Incorporating the Exploration Coefficient.}
Using the definition of $C_1$ from Assumption~\ref{asm:main-early}, we substitute $\mathrm{Var}_r(p(r) u_r(r)) = C_1 \E_r[p(r)^2 \norm{u_r(r)}^2]$:
\begin{align*}
    \Delta &= \E_r \Big[ p(r)(1-p(r)) \qty(\norm{u_r(r)}^2 - \norm{u_j(o^* \given r)}^2) + C_1 p(r)^2 \norm{u_r(r)}^2 \Big].
\end{align*}

\textbf{3. Verification of Positivity.}
To show $\Delta > 0$, we analyze the term inside the expectation (the integrand). We split the expression into multiple lines to isolate the quadratic components:
\begin{align*}
    \text{Integrand} &= (p(r) - p(r)^2)\norm{u_r(r)}^2  - (p(r) - p(r)^2)\norm{u_j(o^* \given r)}^2  + C_1 p(r)^2 \norm{u_r(r)}^2 \\
    &= p(r) \Big[ (1-p(r))\norm{u_r(r)}^2  - (1-p(r))\norm{u_j(o^* \given r)}^2 + C_1 p(r) \norm{u_r(r)}^2 \Big] \\
    &= p(r) \Big[ \norm{u_r(r)}^2 (1 - p(r) + C_1 p(r)) - \norm{u_j(o^* \given r)}^2 (1-p(r)) \Big].
\end{align*}
We now invoke the inequality from Assumption~\ref{asm:main-early}:
\begin{equation*}
    \frac{\norm{u_r(r)}}{\norm{u_j(o^* \given r)}} > \sqrt{\frac{1-p(r)}{1-p(r)+C_1 p(r)}}.
\end{equation*}
Squaring both sides and rearranging:
\begin{equation*}
    \norm{u_r(r)}^2 (1 - p(r) + C_1 p(r)) > \norm{u_j(o^* \given r)}^2 (1-p(r)).
\end{equation*}
This implies that the term inside the square brackets is strictly positive. Since $p(r) \in (0, 1)$, the entire integrand is strictly positive for any $r$. Therefore, the expectation is strictly positive:
\begin{equation*}
    \Delta > 0 \implies \mathrm{Var}(\hat g_B) > \E_{\bar r}[\mathrm{Var}(\hat g_A \given \bar r)].
\end{equation*}
This concludes the proof.
\end{proof}

\section{Implementation Details}
Table~\ref{tab:hp-rl} and Table~\ref{tab:hp-policy} show the hyperparameters used in {\name} and policy model training. We implement the GRPO training based on ms-swift$\footnote{\url{https://github.com/modelscope/ms-swift}}$ library~\citep{msswift} and implement DPO and IterDPO based on LLaMA-Factory$\footnote{\url{https://github.com/hiyouga/LlamaFactory}}$~\citep{llamafactory}. We totally conduct 3 iterations for {\name} alternating RL training. Additionally, the sampling parameters used in inference are summarized in Table~\ref{tab:hp-sampling}. We used the same sampling parameters as their official implementations and papers for baseline methods.
\begin{table*}[h!]
\centering
\caption{Hyper-parameters used in {\name} training. \label{tab:hp-rl}}
\begin{tabular}{@{}c|c|c|c|c|c@{}}
\toprule
                              Module     & Parameter                       & Value &   Module                      & Parameter                       & Value \\ \midrule
\multirow{10}{*}{Rubric Generator} & \#generations                & 6     & \multirow{10}{*}{Judge} & \#generations                & 7     \\
                                   & Cutoff Length         & 512   &                         & Cutoff Length         & 1024  \\
                                   & Batch Size & 288    &                         &Batch Size & 224     \\
                                   & Optimizer   & AdamW     &                         & Optimizer   & AdamW     \\
                                   & Learning Rate                  & 1e-6  &                         & Learning Rate                  & 1e-6  \\
                                   & Temperature                     & 1.0   &                         & Temperature                     & 1.0   \\
                                   & \#iterations                 & 2     &                         & \#iterations                 & 2     \\
                                   & Epochs              & 1     &                         & Epochs              & 1     \\
                                   & $\epsilon_{\text{high}}$                  & 0.28  &                         & $\epsilon_{\text{high}}$                   & 0.28  \\
                                   & $\epsilon_{\text{low}}$                  & 0.2  &                         & $\epsilon_{\text{low}}$                   & 0.2  \\
                                   & $\beta$                            & 0.001 &                         & $\beta$                            & 0.001 \\ \bottomrule
\end{tabular}
\end{table*}

\begin{table*}[h!]
\centering
\caption{Hyper-parameters used in policy model training. \label{tab:hp-policy}}
\begin{tabular}{@{}c|c|c|c|c|c@{}}
\toprule
Method                & Parameter         & Value & Method                 & Parameter     & Value \\ \midrule
\multirow{10}{*}{DPO} & Cutoff Length     & 2048  & \multirow{10}{*}{GRPO} & \#generations & 6     \\
                      & Batch Size        & 64    &                        & Cutoff Length & 2048  \\
                      & Optimizer         & AdamW &                        & Batch Size    & 288   \\
                      & Learning Rate     & 8e-7  &                        & Optimizer     & AdamW \\
                      & Epochs            & 1     &                        & Learning Rate & 5e-7  \\
                      & beta              & 0.1   &                        & Temperature   & 1.0   \\
                      & SFT mixing weight & 0.2   &                        & \#iterations  & 2     \\
                      & /                 & /     &                        & Epochs        & 1     \\
                      & /                 & /     &                        & $\epsilon_{\text{high}}$           & 0.28  \\
                       & /                 & /     &                        & $\epsilon_{\text{low}}$           & 0.2  \\
                      & /                 & /     &                        & $\beta$          & 0.001 \\ \bottomrule
\end{tabular}
\end{table*}

\begin{table*}[h!]
\centering
\caption{Sampling parameters used in {\name} inference. \label{tab:hp-sampling}}
\begin{tabular}{@{}c|c|c|c|c|c@{}}
\toprule
Module                            & Parameter       & Value &      Module                  & Parameter       & Value \\ \midrule
\multirow{5}{*}{Rubric Generator} & Maximum Tokens  & 1024  & \multirow{5}{*}{Judge} & Maximum Tokens  & 4096  \\
                                  & Temperature     & 0.0   &                        & Temperature     & 1.0   \\
                                  & Top-P           & /     &                        & Top-P           & 1.0   \\
                                  & Top-K           & /     &                        & Top-K           & -1    \\
                                  & Enable-thinking & False &                        & Enable-thinking & False \\ \bottomrule
\end{tabular}
\end{table*}

\section{Additional Experimental Results}
\subsection{Performance on WritingPreferenceBench}
We present the performance on WritingPreferenceBench in Table~\ref{tab:wpb}.
\begin{table*}[h!]
\centering
\caption{Comparison of different judge and reward models on WritingPreferenceBench. 
Best results are highlighted in \textbf{bold}. \label{tab:wpb}}
\resizebox{\textwidth}{!}{%
\begin{tabular}{@{}lccccccccc@{}}
\toprule
                  & Func. & Promo. & Non-Fic. & Fiction & Funny & Poetry & Script & Role & AVG  \\ \midrule
\multicolumn{10}{l}{\it LLM as Judge (black-box model)}                                                                                                                                                                                                                                      \\ \midrule
Claude-4-Opus-thinking & 65.7  & 64.3   & 64.1     & 60.1    & 54.2  & 64.0   & 43.5   & 51.7 & 61.0 \\ 
OpenAI-o4-mini         & 58.3  & 58.6   & 60.9     & 55.5    & 53.2  & 68.0   & 30.4   & 55.2 & 56.6 \\ 
Gemini-2.5-Flash       & 59.1  & 57.7   & 62.5     & 59.8    & 52.2  & 56.0   & 34.8   & 51.7 & 57.5 \\ \midrule
\multicolumn{10}{l}{\it White-box Reward Models}                                                                                                                                                                                                                                             \\ \midrule
Skywork-Llama-3.1-8B   & 53.6                      & 56.3                       & 60.6                         & 49.0                        & 52.2                      & 56.0                       & \bf 65.2                       & 41.4                     & 53.1                     \\ 
Skywork-Gemma-2-27B    & 49.0                      & 53.9                       & 59.6                         & 33.9                        & 55.1                      & 36.0                       & 21.7                       & 51.7                     & 46.8                     \\
RM-R1-DeepSeek-Qwen-7B & 62.5                      & 55.1                       & 59.2                         & 55.4                        & 58.0                      & 56.0                       & \bf 65.2                       & 41.4                     & 57.4                     \\ 
RM-R1-Qwen2.5-7B       & 67.0                      & 57.2                       & 53.9                         & 60.0                        & 54.6                      & 72.0                       & 47.8                       &\bf 65.5                     & 59.8                     \\ 
RRM-7B                    & 50.0  & 35.3 & 50.0              & 49.5     & 38.5    & 36.4   & 45.5   & 53.8 & 44.7 \\ \midrule

\multicolumn{10}{l}{\it Rubric-based Models}                                                                                                                                                                                                                                             \\ \midrule
{\textsc{Rubric-RM}} & 58.3 & 58.5 & 57.9 & 58.3 & 58.0 & 76.0 & 47.8 & 55.2 & 60.3 \\ 

\rowcolor{purple!10}
{\name}                   & \bf 67.8                      &\bf 63.1                       &\bf 65.8                         &\bf 60.9                        &\bf 61.0                      &\bf 80.0                       & 47.8                       & 55.2                     &\bf 63.2                     \\ \bottomrule
\end{tabular}%
}
\end{table*}

\subsection{Position Bias Analysis}
\label{app:position_bias}

In this section, we study position bias in pairwise judge and reward models, where the predicted preference may depend on the relative order of the two responses~\citep{shi2025judging}.
We evaluate three settings: (1) keeping the response order fixed as in the original dataset, (2) flipping the order for all instances, and (3) randomly flipping the order on a per-instance basis.
Table~\ref{tab:position} reports results on RewardBench and the IF evaluation benchmarks.
Overall, baseline methods exhibit non-trivial position bias.
For {RRM-7B}, changing the order leads to a {46.2-point} difference on {PPE-IFEval} ({75.8} vs.\ {29.6}).
Likewise, for {RM-R1-7B} (Qwen-2.5-Inst), flipping the order changes {InfoBench} by {11.9 points} ({81.8} vs.\ {69.9}).
For {RM-R1-7B} (DeepSeek-Dist), the order sensitivity remains substantial, with a {9.9-point} difference on {InfoBench} ({78.3} vs.\ {68.4}) and a {9.3-point} difference on {FollowBench} ({79.0} vs.\ {69.7}).
In contrast, our \name\ remains consistently stable across different orderings, suggesting substantially reduced position bias and more robust evaluation.
{This design choice is aligned with our RL training design, where we randomize the response order when collecting reward signals, which further mitigates position bias in downstream policy optimization.}

\begin{table*}[!t]
\centering
\caption{
Position bias analysis for different judge and reward models.
{\name} shows much lower sensitivity to the ordering of response pairs.
}
\label{tab:position}
\resizebox{0.9\textwidth}{!}{%
\begin{tabular}{l cc cccc c}
\toprule
\multirow{2.5}{*}{} 
& \multicolumn{2}{c}{\bf RewardBench} 
& \multicolumn{4}{c}{\bf IF Evaluation Benchmarks} 
& \multirow{2.5}{*}{\bf Avg. Variation} \\ 
\cmidrule(lr){2-3}  \cmidrule(lr){4-7} 
                            & Chat & Chat Hard  & FollowBench   &   PPE-IFEval  &   InfoBench   &   IFBench     &      \\ 
\midrule
\multicolumn{8}{l}{\it White-box Judge/Reward LLM: RRM-7B} \\
\midrule
Mixed Ord	                & 77.7 & 69.5       & 65.5          & 51.0          & 68.2	        & 53.2	        &      \\
Fixed Ord-1	                & 73.9 & 61.6       & 53.8          & 29.6          & 62.3          & 30.2          &      \\
Fixed Ord-2                 & 82.1 & 72.1       & 64.7          & 75.8          & 74.2          & 74.2          &      \\
\cmidrule(lr){2-8}
Variation                   & 8.2  & 10.5       & 11.7          & 46.2          & 11.9          & 44.0          & 22.08 \\

\midrule
\multicolumn{8}{l}{\it White-box Judge/Reward LLM: RM-R1-7B (Qwen-2.5-Inst) } \\
\midrule
Mixed Ord                   & 83.0 & 70.0       & 56.3          & 55.2	        & 71.3	        & 55.2          &      \\
Fixed Ord-1	                & 82.1 & 63.4       & 57.1          & 54.8          & 81.8          & 53.8          &      \\
Fixed Ord-2                 & 82.4 & 71.1       & 56.3          & 50.4          & 69.9          & 54.1          &      \\
\cmidrule(lr){2-8}
Variation                   & 0.9  & 7.7        & 0.8           & 4.8           & 11.9          & 1.4           & 4.58  \\

\midrule
\multicolumn{8}{l}{\it White-box Judge/Reward LLM: RM-R1-7B  (DeepSeek-Dist) } \\
\midrule
Mixed Ord                   & 85.3 & 67.3       & 69.7          & 51.0          & 70.3          & 56.5          & \\
Fixed Ord-1	                & 87.1 & 67.3       & 79.0          & 52.8          & 78.3          & 53.2          & \\
Fixed Ord-2                 & 82.7 & 69.5       & 70.6          & 54.7          & 68.4          & 60.6          & \\
\cmidrule(lr){2-8}
Variation                   & 4.4  & 2.2        & 9.3           & 3.7           & 9.9           & 7.4           & 6.15 \\

\midrule
\multicolumn{8}{l}{\it Rubric-based Method: \textsc{Rubric-RM}} \\
\midrule
Mixed Ord                   & 88.2 & 74.1       & 76.1          & 67.0          & 80.8          & 65.4          &      \\
Fixed Ord-1	                & 87.4 & 74.6       & 79.8          & 70.8          & 80.9          & 66.4          &      \\
Fixed Ord-2                 & 88.7 & 73.5       & 75.6          & 67.2          & 78.5          & 64.4          &      \\
\cmidrule(lr){2-8}
Variation                   & 1.3  & 1.1        & 4.2           & 3.8           & 2.4           & 2.0           & 2.47  \\

\midrule
\multicolumn{8}{l}{\it Rubric-based Method: \name\ (Ours)} \\
\midrule
Mixed Ord                   & 89.4 & 79.6       & 85.7          & 70.8          & 86.1          & 65.9          &      \\
Fixed Ord-1	                & 89.9 & 79.4       & 84.9          & 71.8          & 86.1          & 65.3          &      \\
Fixed Ord-2                 & 88.4 & 80.3       & 85.7          & 71.0          & 87.9          & 66.9          &      \\
\cmidrule(lr){2-8}
Variation                   & 1.5  & 0.9        & 0.8           & 1.0           & 1.8           & 1.6           & 1.27  \\

\bottomrule
\end{tabular}
}
\end{table*}

\subsection{Additional Case Study}
\label{app:case}

In this section 
we compare {\name} with \textsc{Rubric-RM}, another rubric-based RM trained with SFT, on a randomly chosen example from IFBench. The case specifies keywords and paragraph length. Results are shown in Table~\ref{tab:case-if}. 
In this IFBench example, which requires specific keywords and exactly two paragraphs, the baseline \textsc{Rubric-RM} suffers from a judging hallucination, incorrectly claiming that a valid response is split into three paragraphs.
{\name}, on the contrary, accurately extracts these hard constraints and identifies the missing \emph{open-source} keyword in the negative sample, while correctly verifying the structure of the positive one.

\input{tables/case_if_rubricrms}

\section{Prompts}
\label{app:prompts}

We present the prompts we used in this section. 
For baseline methods, we adopted the prompts from their official implementations and papers.

\input{tables/prompt_rubric}

%% file: tables/case_if_rubricrms.tex
\begin{table*}[!t]
\centering
\begingroup
\footnotesize
\setlength{\tabcolsep}{4pt}
\definecolor{FailBg}{RGB}{253,219,219}
\definecolor{PassBg}{RGB}{214,240,221}
\newcommand{\bad}[1]{\textcolor{red}{#1}}
\newcommand{\good}[1]{\textcolor{green!60!black}{#1}}
\caption{
Case study with error highlighting. In the baseline \textsc{Rubric-RM}, the rubric output confused the judge and led to an incorrect prediction.\label{tab:case-if}}
\begin{tabular}{@{}p{0.15\linewidth}p{0.84\linewidth}@{}}
\toprule
\multicolumn{2}{@{}l}{\textbf{Case Study on IFBench}}\\
\midrule
\textbf{Prompt} &
What is Nextcloud, why should I use it, and how does it compare to other cloud storage solutions? Please provide the response with a tone of excitement and enthusiasm. {Your response should contain \emph{exactly 2 paragraphs} and include the keywords `cloud storage' and `open-source'.}
\\[2pt]

\textbf{Resp A (snippet)} &
``Nextcloud is an incredibly exciting and innovative {cloud storage} solution \dots As an {open-source} platform \dots Imagine having your very own private cloud \dots'' (\emph{2 paragraphs})
\\[2pt]

\textbf{Resp B (snippet)} &
``Nextcloud is an incredibly powerful and flexible solution \dots While services like Google Drive and Dropbox are great \dots'' (\emph{2 paragraphs; missing ``open-source''})
\\

\textbf{Label} & Resp A.\\

\midrule
\multicolumn{2}{@{}l}{\textbf{\textsc{Rubric-RM}}}\\
\midrule

\textbf{Rubric} &
1. The response must define what Nextcloud is. [Hard Rule] 2. The response must explain why the user should use Nextcloud. [Hard Rule] 3. The response must compare Nextcloud to other cloud storage solutions. [Hard Rule] 4. The response must use a tone of excitement and enthusiasm. [Hard Rule] 5. The response must consist of exactly two paragraphs. [Hard Rule] 6. The response must include the keywords `cloud storage' and `open-source'. [Hard Rule] 7. The response should present information in a clear, logical structure that guides the reader through definition, benefits, and comparison. [Principle] 8. The response should use vivid, engaging language to maintain reader interest. [Principle] 9. The response should employ precise and varied vocabulary to convey ideas effectively. [Principle] 10. The response should provide specific, relevant details to substantiate its claims. [Principle] 11. The response should maintain coherence and flow, ensuring each paragraph builds on the previous one. [Principle] 12. The response should balance conciseness with sufficient elaboration to cover all required points. [Principle]
\\
\cmidrule{2-2}

\textbf{Judge} &
``\ldots Response A is \bad{split into three paragraphs, not exactly two} \ldots Response A fails the gatekeeper criterion (paragraph count) \ldots \bad{(Choose B)}
\\

\midrule
\multicolumn{2}{@{}l}{\textbf{{\name}}}\\
\midrule

\textbf{Rubric} &
1. The response must address what Nextcloud is, why it should be used, and how it compares to other cloud storage solutions. [Hard Rule] 2. The response must maintain an enthusiastic and excited tone throughout. [Hard Rule] 3. The response must consist of exactly two paragraphs. [Hard Rule] 4. The response must include the keywords `cloud storage' and `open-source'. [Hard Rule] 5. The response should present information in a clear and logical structure, grouping related ideas coherently. [Principle] 6. The response should use vivid and engaging language to sustain reader interest. [Principle] 7. The response should provide relevant comparisons or contrasts to contextualize the subject. [Principle] 8. The response should maintain conciseness and avoid unnecessary repetition. [Principle]
\\
\cmidrule{2-2}

\textbf{Judge} &
``\ldots Response B mentions `cloud' and `open' but \good{not the exact keyword `open-source,' so fails.} \ldots Response A fulfills all hard rules, including the precise keywords \ldots \good{(Choose A)}'' 
\\

\bottomrule
\end{tabular}

\endgroup
\end{table*}

%% file: tables/prompt_rubric.tex
\UseRawInputEncoding

\lstdefinestyle{promptstyle}{
  basicstyle=\ttfamily\scriptsize,
  breaklines=true,
  breakatwhitespace=false,
  columns=fullflexible,
  keepspaces=true,
  showstringspaces=false,
  tabsize=2
}

\begin{tcblisting}{
  listing only,
  breakable,
  enhanced,
  colback=purple!5,
  colframe=purple!50,
  rounded corners,top=6pt,bottom=6pt,left=8pt,right=8pt,boxsep=4pt,
  title=Prompt for Rubric Generation ({\name}),
  listing options={style=promptstyle}
}
Your task is to extract a set of rubric-style instructions from a user's request.
These rubrics will be used as evaluation criteria to check if a response fully meets the request.
Every rubric item must be a universal principle. If any rubric still contains topic-specific references (e.g., names, places, myths, numbers, historical facts), it is automatically invalid.

- **Two Distinct Categories:**
  - [Hard Rule]: Derived strictly from explicit requirements stated in the <request> (format, length, structure, forbidden/required elements, etc.).
  - [Principle]: Derived by abstracting any concrete cues into domain-agnostic quality criteria (e.g., clarity, correctness, sound reasoning, pedagogy).

- **Comprehensiveness:**
  The rubric must cover all critical aspects implied by the request and examples, including explicit requirements and implicit quality standards.

- **Conciseness & Uniqueness:**
  Each rubric must capture a distinct evaluation criterion. Overlapping or redundant criteria must be merged into a single rubric. Wording must be precise and free of repetition.

- **Format Requirements:**
  - Use a numbered list.
  - Each item starts with "The response" phrased in third person.
  - Append [Hard Rule] or [Principle] at the end of each item.
  - Do not include reasoning, explanations, or examples in the final output—only the rubrics.

Here is the request:
{prompt}

Please generate the rubrics for the above request.
\end{tcblisting}

\begin{tcblisting}{
  listing only,
  enhanced,
  colback=purple!5,
  colframe=purple!50,
  rounded corners,
  top=6pt,bottom=6pt,left=8pt,right=8pt,boxsep=4pt,
  title=Prompt for Judge Generation ({\name}),
  listing options={style=promptstyle}
}
You are a fair and impartial judge. Your task is to evaluate 'Response A' and 'Response B' based on a given instruction and a rubric. You will conduct this evaluation in distinct phases as outlined below.

### Phase 1: Compliance Check Instructions
First, identify the single most important, objective 'Gatekeeper Criterion' from the rubric.
- **A rule is objective (and likely a Gatekeeper) if it can be verified without opinion. Key examples are: word/paragraph limits, required output format (e.g., JSON validity), required/forbidden sections, or forbidden content.**
- **Conversely, a rule is subjective if it requires interpretation or qualitative judgment. Subjective rules about quality are NOT Gatekeepers. Examples include criteria like "be creative," "write clearly," "be engaging," or "use a professional tone."**

### Phase 2: Analyze Each Response
Next, for each Gatekeeper Criterion and all other criteria in the rubric, evaluate each response item by item.

### Phase 3: Final Judgment Instructions
Based on the results from the previous phases, determine the winner using these simple rules. Provide a final justification explaining your decision first and then give your decision.

---
### REQUIRED OUTPUT FORMAT
You must follow this exact output format below.

--- Compliance Check ---
Identified Gatekeeper Criterion: <e.g., Criterion 1: Must be under 50 words.>

--- Analysis ---
**Response A:**
- Criterion 1 [Hard Rule]: Justification: <...>
- Criterion 2 [Hard Rule]: Justification: <...>
- Criterion 3 [Principle]: Justification: <...>
- ... (and so on for all other criteria)

**Response B:**
- Criterion 1 [Hard Rule]: Justification: <...>
- Criterion 2 [Hard Rule]: Justification: <...>
- Criterion 3 [Principle]: Justification: <...>
- ... (and so on for all other criteria)

--- Final Judgment ---
Justification: <...>
Winner: <Response A / Response B>

Task to Evaluate:
Instruction:
{instruction}

Rubric:
{rubric}

Response A:
{response_a}

Response B:
{response_b}
\end{tcblisting}